%% file: sample.tex
\documentclass[twoside,11pt]{article}

\usepackage{blindtext}

\usepackage[preprint, abbrvbib]{jmlr2e}

\usepackage{lastpage}

\ShortHeadings{Global Convergence of Average Reward CMDPs with Neural Critic Parameterization}{Global Convergence of Average Reward CMDPs with Neural Critic Parameterization}
\firstpageno{1}

\usepackage{amsmath}
\usepackage{amssymb}
\usepackage{pifont}
\usepackage{amsfonts}
\usepackage{algorithm}
\usepackage{algorithmic}
\usepackage{tcolorbox}
\usepackage{tablefootnote}
\usepackage{booktabs}
\usepackage{tikz}
\usepackage{natbib}
\newtheorem{assumption}[theorem]{Assumption}
\newcommand{\cmark}{\ding{51}}
\newcommand{\xmark}{\ding{55}}
\allowdisplaybreaks

\begin{document}

\title{Global Convergence of Average Reward Constrained MDPs with Neural Critic and General Policy Parameterization}

\author{\name Anirudh Satheesh \email anirudhs@terpmail.umd.edu \\
       \addr Department of Computer Science\\
       University of Maryland\\
       College Park, MD 20740, USA
       \AND
       \name Pankaj Kumar Barman \email pankajkb24@iitk.ac.in \\
       \addr 
       Indian Institute of Technology, Kanpur\\
       Uttar Pradesh, 208016, India
       \AND
       \name Washim Uddin Mondal \email wmondal@iitk.ac.in \\
       \addr Department of Electrical Engineering\\
       Indian Institute of Technology, Kanpur\\
       Uttar Pradesh, 208016, India
       \AND
       \name Vaneet Aggarwal \email vaneet@purdue.edu \\
       \addr Department of Industrial Engineering\\
       Purdue University\\
       West Lafayette, IN 47907, USA}

\editor{My editor}

\maketitle

\begin{abstract}
We study infinite-horizon Constrained Markov Decision Processes (CMDPs) with general policy parameterizations and multi-layer neural network critics. Existing theoretical analyses for constrained reinforcement learning largely rely on tabular policies or linear critics, which limits their applicability to high-dimensional and continuous control problems. We propose a primal–dual natural actor–critic algorithm that integrates neural critic estimation with natural policy gradient updates and leverages Neural Tangent Kernel (NTK) theory to control function-approximation error under Markovian sampling, without requiring access to mixing-time oracles. We establish global convergence and cumulative constraint violation rates of $\tilde{\mathcal{O}}(T^{-1/4})$ up to approximation errors induced by the policy and critic classes. Our results provide the first such guarantees for CMDPs with general policies and multi-layer neural critics, substantially extending the theoretical foundations of actor–critic methods beyond the linear-critic regime.
\end{abstract}


\input{Sections/Introduction}
\input{Sections/RelatedWork}
\input{Sections/formulation}
\input{Sections/algorithm}
\input{Sections/assumptions}
\input{Sections/theoretical_analysis}
\input{Sections/conclusion}


\bibliography{sample}
\newpage

\appendix
\input{Sections/Appendix/related_work}
\input{Sections/Appendix/mlmc_estimation}
\input{Sections/Appendix/neural_critic_analysis}
\input{Sections/Appendix/npg_analysis}
\input{Sections/Appendix/cmdp_analysis}
\input{Sections/Appendix/limitations_future_work}

\end{document}

%% file: Sections/Introduction.tex
\section{INTRODUCTION}

Reinforcement Learning (RL) has achieved remarkable empirical success in solving complex, high-dimensional continuous control problems, an advancement largely driven by the representational power of deep neural networks. However, as RL is increasingly deployed in safety-critical applications such as transportation \citep{al2019deeppool}, healthcare \citep{tamboli2024reinforced}, and robotics \citep{Gonzalez2023}, it is imperative to ensure that these agents adhere to strict operational constraints. This requirement is naturally formulated as a Constrained Markov Decision Process (CMDP), wherein the objective is to maximize a primary reward signal while maintaining auxiliary cost signals below predefined thresholds. To solve continuous-space CMDPs, primal-dual actor-critic algorithms have emerged as a dominant paradigm. In these methods, the actor updates a parameterized policy via gradient ascent, a critic evaluates the policy's performance to construct the gradient, and a dual variable dynamically penalizes constraint violations. 

Despite the widespread practical adoption of deep actor-critic methods, the theoretical understanding of these algorithms remains limited. Early theoretical frameworks establishing global convergence for actor-critic algorithms were predominantly restricted to tabular settings or linear function approximations \citep{chen2022learning, bai2023achieving, xu2025global}. While linear models provide mathematical tractability, they fail to capture the complex, non-linear feature abstractions that characterize modern deep RL. More recently, literature has begun to bridge this gap by analyzing neural network parameterizations through the lens of Neural Tangent Kernel (NTK) theory \citep{gaur2024closing, ganesh2025orderoptimal}. However, these theoretical advancements have largely been confined to unconstrained Markov Decision Processes or discounted reward formulations, leaving the average reward constrained setting unexplored. Thus, we ask the following question:

\begin{tcolorbox}[
    colback=white,
    colframe=black,
    arc=3mm,
    boxrule=1.5pt,
    left=6pt,
    right=6pt,
    top=6pt,
    bottom=6pt
]
\textit{Can we design a primal-dual actor-critic algorithm with general policy parameterizations and multi-layer neural network critics that achieves provable global convergence for average-reward CMDPs under Markovian sampling, without requiring a mixing-time oracle?}
\end{tcolorbox}
\begin{table*}[ht]
\centering
\caption{Comparison of relevant algorithms for infinite-horizon
average-reward Reinforcement Learning in the model-free setting. Our work is the first to provide average reward global convergence and constraint violation results for constrained MDPs with general policy and multi-layer neural network critic parametrization. We provide a more detailed comparison to prior work in Appendix~\ref{appendix: related work}.}
\label{tbl_related1}%
\resizebox{\textwidth}{!}{
\begin{tabular}{|c|c|c|c|c|c|c|c|}
\hline
References &  
\begin{tabular}[c]{@{}c@{}} Global \\ Convergence \end{tabular} & 
\begin{tabular}[c]{@{}c@{}} Violation \end{tabular} & 
\begin{tabular}[c]{@{}c@{}} Reward \\ Setting \end{tabular} &
\begin{tabular}[c]{@{}c@{}} Infinite State \\ Action Space \end{tabular} &   
\begin{tabular}[c]{@{}c@{}} Mixing Time \\ Unknown? \end{tabular} &   
\begin{tabular}[c]{@{}c@{}} Critic \\ Parameterization  \end{tabular} &   
\begin{tabular}[c]{@{}c@{}} Policy \\ Parameterization \end{tabular}  \\ 
\hline 
\citep{cayci2024finite}     
& $\tilde{\mathcal{O}}(T^{-1/5})$ 
& N/A 
& Discounted
& \textcolor{red}{\xmark} 
& \textcolor{green}{\cmark} 
& Multi layer NN 
& Multi layer NN \\ 

\citep{gaur2024closing} 
& $\tilde{\mathcal{O}}(T^{-1/3})$ 
& N/A 
& Discounted
& \textcolor{green}{\cmark} 
& \textcolor{green}{\cmark} 
& Multi layer NN 
& General \\ 

\citep{ganesh2025orderoptimal} 
& $\tilde{\mathcal{O}}(T^{-1/2})$ 
& N/A 
& Discounted
& \textcolor{green}{\cmark} 
& \textcolor{red}{\xmark} 
& Multi layer NN 
& General \\ 

\citep{mondal2024sample} 
& $\tilde{\mathcal{O}}(T^{-1/2})$ 
& $\tilde{\mathcal{O}}(T^{-1/2})$ 
& Discounted
& \textcolor{green}{\cmark} 
& \textcolor{green}{\cmark} 
& N/A 
& General \\ 

\citep{agarwal2022regret,agarwal2022concave} 
& $\tilde{\mathcal{O}}(T^{-1/2})$ 
& 0
& Average 
& \textcolor{red}{\xmark} 
& \textcolor{red}{\xmark} 
& Tabular 
& Tabular \\ 

\citep{ghosh2022achieving}       
& $\tilde{\mathcal{O}}(T^{-1/2})$ 
& $\tilde{\mathcal{O}}(T^{-1/2})$ 
& Average
& \textcolor{green}{\cmark} 
& \textcolor{red}{\xmark} 
& Linear MDP 
& Tabular \\                                                              

\citep{bai2024learning} 
& $\tilde{\mathcal{O}}(T^{-1/5})$ 
& $\tilde{\mathcal{O}}(T^{-1/5})$ 
& Average
& \textcolor{green}{\cmark} 
& \textcolor{red}{\xmark} 
& N/A 
& General \\

\citep{xu2025global}    
& $\tilde{\mathcal{O}}(T^{-1/2})$ 
& $\tilde{\mathcal{O}}(T^{-1/2})$ 
& Average
& \textcolor{green}{\cmark} 
& \textcolor{green}{\cmark} 
& Linear 
& General \\

\textbf{This work} 
& $\tilde{\mathcal{O}}(T^{-1/4})$ 
& $\tilde{\mathcal{O}}(T^{-1/4})$ 
& Average
& \textcolor{green}{\cmark} 
& \textcolor{green}{\cmark} 
& Multi layer NN 
& General \\

\hline
\end{tabular}
}
\end{table*}
\subsection{Challenges and Contributions}

Addressing this problem requires overcoming three central technical obstacles. First, handling Markovian sampling dependence is typically achieved through data dropping strategies that rely on a mixing time oracle to specify appropriate discard intervals. Such assumptions are restrictive and often impractical. Second, analyzing a multi-layer neural critic in the Neural Tangent Kernel (NTK) regime demands precise control of its approximation bias. If this bias does not decay sufficiently fast, the resulting Natural Policy Gradient (NPG) updates become inaccurate and unstable. Third, the infinite-horizon average-reward setting presents a fundamental difficulty: unlike discounted formulations, the average-reward Bellman operator is not contractive. This lack of contraction destabilizes critic evaluation. When combined with the saddle-point structure of constrained Markov decision processes (CMDPs), the coupled estimation errors from the actor, critic, and dual variables can accumulate and potentially lead to divergence.

To address these challenges, we introduce the Primal-Dual Natural Actor-Critic with Neural Critic (PDNAC-NC) algorithm and establish its global convergence guarantees. To remove the dependence on a mixing-time oracle, we incorporate a Multi-Level Monte Carlo (MLMC) estimator that samples trajectory lengths from a geometric distribution. This construction corrects Markovian bias without discarding data. To manage the nonlinear neural critic, we restrict its parameters to remain within an NTK neighborhood around initialization and prove that its approximation bias decays proportionally to the mean-squared error, ensuring reliable NPG estimation. Finally, to handle the non-contractive Bellman operator and the primal-dual structure, we develop a refined coupled analysis that carefully tracks error propagation across actor, critic, and dual updates. This approach controls error amplification throughout the min-max dynamics and yields a convergence rate of \(\tilde{\mathcal{O}}(T^{-1/4})\) in both optimality gap and constraint violation. To our knowledge, this is the first result establishing global convergence for average-reward CMDPs with multi-layer neural critics and general policy parameterizations, without assuming access to a mixing-time oracle.

%% file: Sections/RelatedWork.tex
\section{Related works}

\textbf{Neural Actor-Critic and Markovian Sampling:} Recent theoretical advancements have begun to bridge the gap between deep reinforcement learning practice and theory by analyzing actor-critic methods equipped with multi-layer neural networks via Neural Tangent Kernel (NTK) theory \citep{fu2020single, cayci2024finite, tian2023convergence, gaur2024closing, ganesh2025orderoptimal}. While these contemporary works achieve state-of-the-art sample complexities for general policy parameterizations, their analyses are strictly confined to standard, unconstrained MDPs with discounted rewards. Furthermore, under Markovian sampling, these approaches rely on restrictive data-dropping techniques that discard collected transitions based on a known mixing-time oracle to mitigate statistical dependencies. However, this results in only one out of every \(\tau_{\mathrm{mix}}\) samples used for policy updates. To bypass this, recent theoretical works have adapted Multi-Level Monte Carlo (MLMC) estimation to RL \citep{beznosikov2023first, xu2025global}, providing unbiased gradient estimates that correct for Markovian bias without sample thinning. However, these MLMC techniques have previously been restricted to algorithms utilizing linear function approximations for the critic. We are the first to extend this technique to the neural critic setting, mitigating the issues of data dropping.

\textbf{Constrained Average-Reward MDPs:} For safe reinforcement learning, infinite-horizon average-reward Constrained MDPs (CMDPs) present a difficult challenge due to coupled primal-dual dynamics and the lack of Bellman operator contractivity. Consequently, existing theoretical guarantees for average-reward CMDPs have predominantly relied on tabular settings, exact gradients, linear MDP, or linear critic approximations \citep{paternain2019constrained,ding2020natural,chen2022learning, bai2024learning,xu2025global,ghosh2022achieving, wei2022provably,satheesh2026regret}, failing to capture the complex, non-linear feature representations characteristic of deep RL. Our work distinguishes itself by being the first to theoretically analyze a primal-dual natural actor-critic algorithm for average-reward CMDPs utilizing multi-layer neural network critics. By uniquely integrating MLMC within a nested-loop structure to simultaneously evaluate a neural critic and estimate the Natural Policy Gradient (NPG), our approach eliminates the need for a mixing-time oracle while achieving global convergence in this setting.

%% file: Sections/formulation.tex
\section{Formulation}

Consider an infinite-horizon average reward constrained Markov Decision Process (CMDP) denoted as $\mathcal{M} = (\mathcal{S}, \mathcal{A}, r, c, P, \rho)$ where $\mathcal{S}$ is the state space, $\mathcal{A}$ is the action space of size $A$, $r : \mathcal{S} \times \mathcal{A} \to [0, 1]$ represents the reward function, $c : \mathcal{S} \times \mathcal{A} \to [-1, 1]$ denotes the constraint cost function, $P : \mathcal{S} \times \mathcal{A} \to \Delta(\mathcal{S})$ is the state transition function where $\Delta(\mathcal{S})$ denotes the probability simplex over $\mathcal{S}$, and $\rho \in \Delta(\mathcal{S})$ indicates the initial distribution of states.

A stationary policy $\pi : \mathcal{S} \to \Delta(\mathcal{A})$ maps each state to a distribution over actions. Under policy $\pi$, the agent generates trajectories $\{(s_t, a_t)\}_{t=0}^\infty$ where $a_t \sim \pi(\cdot|s_t)$ and $s_{t+1} \sim P(\cdot|s_t, a_t)$. The policy induces a state transition kernel $P^\pi : \mathcal{S} \to \Delta(\mathcal{S})$ defined by $P^\pi(s, s') = \sum_{a \in \mathcal{A}} \pi(a|s) P(s'|s, a)$ for all $s, s' \in \mathcal{S}$.

We assume the CMDP is ergodic throughout this work.
\begin{assumption}[Ergodicity]\label{assump:ergodic}
The CMDP $\mathcal{M}$ is ergodic, i.e., the Markov chain $\{s_t\}_{t \geq 0}$ induced under every policy $\pi$ is irreducible and aperiodic.
\end{assumption}

Under ergodicity, for each policy $\pi$, there exists a unique $\rho$-independent stationary distribution $d^\pi \in \Delta(\mathcal{S})$ satisfying $P^\pi d^\pi = d^\pi$. We define the state-action occupancy measure as $\nu^\pi(s, a) \triangleq d^\pi(s) \pi(a|s)$ for all $(s, a) \in \mathcal{S} \times \mathcal{A}$.

\begin{definition}[Mixing Time]
The mixing time of the CMDP $\mathcal{M}$ with respect to a policy $\pi$ is defined as
\begin{equation}
\tau_{\mathrm{mix}}^\pi := \min\left\{t \geq 1 \middle| \|(P^\pi)^t(s, \cdot) - d^\pi\|_{\mathrm{TV}} \leq \frac{1}{4},  \forall s\right\},
\end{equation}
where $\|\cdot\|_{\mathrm{TV}}$ denotes the total variation distance. The uniform mixing time is $\tau_{\mathrm{mix}} := \sup_\pi \tau_{\mathrm{mix}}^\pi$, which is finite under ergodicity.
\end{definition}

The average reward and average constraint cost of a policy $\pi$ are defined as
\begin{equation}
J_g^\pi \triangleq \lim_{T \to \infty} \frac{1}{T} \mathbb{E}_\pi\left[\sum_{t=0}^{T-1} g(s_t, a_t) \middle| s_0 \sim \rho\right], \quad g \in \{r, c\},
\end{equation}
where the expectation is over the distribution of $\pi$-induced trajectories. Under ergodicity, $J_g^\pi$ is independent of $\rho$ and can be written as $J_g^\pi = \mathbb{E}_{(s,a) \sim \nu^\pi}[g(s, a)]$. Our objective is to maximize the average reward while ensuring the average cost exceeds a given threshold. Without loss of generality, we formulate this as
\begin{equation}
    \label{eq:cmdp_original}
    \max_\pi J_r^\pi \quad \text{subject to} \quad J_c^\pi \geq 0.
\end{equation}

When the state space $\mathcal{S}$ is large or continuous, we consider a class of parametrized policies $\{\pi_\theta \mid \theta \in \Theta\}$ indexed by a $d$-dimensional parameter $\theta \in \mathbb{R}^d$ where $d \ll |\mathcal{S}||\mathcal{A}|$. Using the shorthand $J_g(\theta) = J_g^{\pi_\theta}$ for $g \in \{r, c\}$, the optimization problem \eqref{eq:cmdp_original} becomes
\begin{equation}
    \label{eq:cmdp_param}
    \max_{\theta \in \Theta} J_r(\theta) \quad \text{subject to} \quad J_c(\theta) \geq 0.
\end{equation}

We assume the optimization problem \eqref{eq:cmdp_param} satisfies the Slater condition, which ensures the existence of an interior point solution.
\begin{assumption}[Slater Condition]\label{assump:slater}
There exists $\delta_s \in (0, 1)$ and $\bar{\theta} \in \Theta$ such that $J_c(\bar{\theta}) \geq \delta_s$.
\end{assumption}
The action-value function corresponding to policy $\pi_\theta$ for $g \in \{r, c\}$ is defined for all $(s, a) \in \mathcal{S} \times \mathcal{A}$ as
\begin{equation}
Q_g^{\pi_\theta}(s, a) = \mathbb{E}_{\pi_\theta}\left[\sum_{t=0}^{\infty} \left(g(s_t, a_t) - J_g(\theta)\right) \middle| s_0 = s, a_0 = a\right].
\end{equation}
The state-value function is $V_g^{\pi_\theta}(s) = \mathbb{E}_{a \sim \pi_\theta(\cdot|s)}[Q_g^{\pi_\theta}(s, a)]$ for all $s \in \mathcal{S}$, and the advantage function is $A_g^{\pi_\theta}(s, a) \triangleq Q_g^{\pi_\theta}(s, a) - V_g^{\pi_\theta}(s)$. The Bellman equation for average reward MDPs takes the form
\begin{equation}\label{eq:bellman}
Q_g^{\pi_\theta}(s, a) = g(s, a) - J_g(\theta) + \mathbb{E}_{s' \sim P(\cdot|s,a)}\left[V_g^{\pi_\theta}(s')\right]
\end{equation}
for all $(s, a) \in \mathcal{S} \times \mathcal{A}$ and $g \in \{r, c\}$. We solve \eqref{eq:cmdp_param} via a primal-dual approach based on the saddle point optimization
\begin{equation}\label{eq:lagrangian}
\max_{\theta \in \Theta} \min_{\lambda \geq 0} L(\theta, \lambda) \triangleq J_r(\theta) + \lambda J_c(\theta),
\end{equation}
where $L(\cdot, \cdot)$ is the Lagrange function and $\lambda \geq 0$ is the dual variable. The policy gradient theorem \cite{sutton1999policy} provides the gradient of the average reward for \(g \in \{r, c\}\):
\begin{equation}\label{eq:policy_gradient}
\nabla_\theta J_g(\theta) = \mathbb{E}_{\substack{s \sim d^{\pi_\theta} \\ a \sim \pi_\theta(\cdot|s)}}\left[A_g^{\pi_\theta}(s, a) \nabla_\theta \log \pi_\theta(a|s)\right].
\end{equation}

Rather than updating along the vanilla gradient, we employ the Natural Policy Gradient (NPG) direction
\begin{equation}
\omega_{g,\theta}^* \triangleq F(\theta)^{-1} \nabla_\theta J_g(\theta),
\end{equation}
where $F(\theta) \in \mathbb{R}^{d \times d}$ is the Fisher information matrix defined as
\begin{equation}
F(\theta) = \mathbb{E}\left[\nabla_\theta \log \pi_\theta(a|s) \otimes \nabla_\theta \log \pi_\theta(a|s)\right],
\end{equation}
with $\otimes$ denoting the outer product and the expectation taken under the state-action occupancy measure. The NPG direction $\omega_{g,\theta}^*$ can equivalently be expressed as the solution to the strongly convex optimization problem
\begin{equation}
    \label{eq:npg_opt}
    \min_{\omega \in \mathbb{R}^d} f_g(\theta, \omega) := \frac{1}{2}\omega^\top F(\theta)\omega - \omega^\top \nabla_\theta J_g(\theta),
\end{equation}
whose gradient is $\nabla_\omega f_g(\theta, \omega) = F(\theta)\omega - \nabla_\theta J_g(\theta)$.

Since the gradients in \eqref{eq:policy_gradient} and \eqref{eq:npg_opt} depend on the unknown action-value function $Q_g^{\pi_\theta}$, we introduce a neural network critic to approximate it. Let $\phi : \mathcal{S} \times \mathcal{A} \to \mathbb{R}^n$ be a fixed feature map satisfying $\|\phi(s, a)\| \leq 1$ for all $(s, a)$. For each $g \in \{r, c\}$, we parameterize the Q-function using an $L$-layer feedforward neural network defined by
\begin{equation}\label{eq:neural_recursion}
x^{(l)} = \frac{1}{\sqrt{m}} \sigma\left(W_l x^{(l-1)}\right), \quad l \in \{1, 2, \ldots, L\},
\end{equation}
where $x^{(0)} = \phi(s, a) \in \mathbb{R}^n$ is the input, $W_1 \in \mathbb{R}^{m \times n}$ and $W_l \in \mathbb{R}^{m \times m}$ for $2 \leq l \leq L$ are weight matrices, $m$ is the network width, and $\sigma(\cdot)$ is the activation function applied element-wise. The approximate Q-function is then
\begin{equation}\label{eq:neural_q}
Q_g(\phi(s, a); \zeta_g) = \frac{1}{\sqrt{m}} b_g^\top x^{(L)},
\end{equation}
where $\zeta_g = (\mathrm{Vec}(W_1); \ldots; \mathrm{Vec}(W_L)) \in \mathbb{R}^{m(n + (L-1)m)}$ collects all trainable weights with $\mathrm{Vec}(\cdot)$ denoting vectorization by stacking columns, and $b_g \in \mathbb{R}^m$ is a fixed random vector. The network is initialized by drawing each entry of $W_l^0$ i.i.d.\ from $\mathcal{N}(0, 1)$ and each entry of $b_g$ uniformly from $\{-1, +1\}$. We denote $\zeta_{g,0} = (\mathrm{Vec}(W_1^0); \ldots; \mathrm{Vec}(W_L^0))$ as the initial parameters.

\begin{assumption}[Smooth Activation]\label{assump:activation}
The activation function $\sigma(\cdot)$ is $L_1$-Lipschitz and $L_2$-smooth, i.e., for all $y_1, y_2 \in \mathbb{R}$: $|\sigma(y_1) - \sigma(y_2)| \leq L_1|y_1 - y_2|$ and $|\sigma'(y_1) - \sigma'(y_2)| \leq L_2|y_1 - y_2|$.
\end{assumption}

Assumption~\ref{assump:activation} is satisfied by smooth approximations of ReLU such as Sigmoid, ELU, and GeLU, which are commonly used in practice \cite{clevert2015fast, hendrycks2016gaussian}. This provides an $O(m^{-1/2})$-smooth property for the neural Q-function. To enable Neural Tangent Kernel (NTK) analysis, we constrain the critic parameters to a ball around initialization:
\begin{equation}
\mathcal{S}_R := \left\{\zeta : \|\zeta - \zeta_0\|_2 \leq R\right\},
\end{equation}
and denote the Euclidean projection onto $\mathcal{S}_R$ as $\Pi_{\mathcal{S}_R}$. The local linearization of the neural Q-network at initialization defines the function class
\begin{align}
\label{eq:linearized_class}
    &\mathcal{F}_{R,m} := \Big\{\hat{Q}_g(\cdot; \zeta) = Q_g(\cdot; \zeta_{g,0}) + \langle \nabla_\zeta Q_g(\cdot; \zeta_{g,0}), \zeta - \zeta_{g,0} \rangle \Big| \zeta \in \mathcal{S}_R \Big\}.
\end{align}

%% file: Sections/algorithm.tex
\section{Algorithm}
\label{sec:algorithm}
We solve \eqref{eq:cmdp_param} via a primal-dual algorithm that updates the policy parameter $\theta$ and dual variable $\lambda$ iteratively. Starting from an arbitrary initial point $(\theta_0, \lambda_0 = 0)$, the ideal updates are
\begin{align}
    \label{eq:ideal_update}
    \theta_{k+1} &= \theta_k + \alpha F(\theta_k)^{-1}\nabla_\theta L(\theta_k, \lambda_k) \\
    \lambda_{k+1} &= \mathcal{P}_{[0, 2/\delta]}\left[\lambda_k - \beta J_c(\theta_k)\right],
\end{align}
where $\alpha$ and $\beta$ are primal and dual learning rates respectively, $\delta$ is the Slater parameter from Assumption~\ref{assump:slater}, and $\mathcal{P}_\Lambda$ denotes projection onto $\Lambda$. Since $\nabla_\theta L(\theta_k, \lambda_k)$, $F(\theta_k)$, and $J_c(\theta_k)$ are not exactly computable due to the unknown transition kernel $P$ and stationary distribution $\nu^{\pi_{\theta_k}}$, we employ approximate updates
\begin{align}\label{eq:approx_update}
\theta_{k+1} = \theta_k + \alpha \omega_k, \quad \lambda_{k+1} = \mathcal{P}_{[0, 2/\delta]}\left[\lambda_k - \beta \eta_c^k\right],
\end{align}
where $\omega_k$ estimates the NPG direction $F(\theta_k)^{-1}\nabla_\theta L(\theta_k, \lambda_k)$ and $\eta_c^k$ estimates $J_c(\theta_k)$. Our algorithm, Primal-Dual Natural Actor-Critic with Neural Critic (PDNAC-NC), is presented in Algorithm~\ref{alg:pdnac_nc}. It employs a Multi-Level Monte Carlo (MLMC) estimation procedure within a nested loop structure where the outer loop of length $K$ updates the primal-dual parameters, and the inner loops of length $H$ execute the neural critic and NPG estimation subroutines.

\subsection{Neural Critic Estimation}
\label{subsec: neural critic estimation}
At the $k$-th epoch, the critic estimation subroutine serves two purposes: (i) obtain an estimate $\eta_g^k$ of the average reward/cost $J_g(\theta_k)$ for $g \in \{r, c\}$, and (ii) obtain the neural network parameters $\zeta_g^k$ that approximate the Q-function $Q_g^{\pi_{\theta_k}}$. For the average reward estimation, note that $J_g(\theta_k)$ solves the optimization
\begin{align}
\label{eq:avg_reward_opt}
\min_{\eta \in \mathbb{R}} \mathcal{R}_g(\theta_k, \eta) = \frac{1}{2} \mathbb{E}_{\nu_g^{\pi_{\theta_k}}} \left(\eta - g(s, a)\right)^2
\end{align}
where the expectation is taken over the state-action occupancy measure. For the Q-function approximation, we minimize the mean-squared projected Bellman error using the linearized neural network class $\mathcal{F}_{R,m}$. Define the critic objective
\begin{align}
\label{eq:critic_obj}
&\mathcal{E}_g(\theta_k, \zeta)\!=\!\frac{1}{2} \mathbb{E}_{\nu_g^{\pi_{\theta_k}}} \left(Q_g(\phi(s,a); \zeta) - Q_g^{\pi_{\theta_k}}(s, a)\right)^2
\end{align}
Thus, by the Bellman equation, the gradients are represented as
\begin{align}
    &\nabla_\eta \mathcal{R}_g(\theta_k, \eta) = \mathbb{E}_{\nu_g^{\pi_{\theta_k}}}\left[\eta - g(s, a)\right] \\
    &\nabla_\zeta \mathcal{E}_g(\theta_k, \zeta) = \mathbb{E}_{\nu_g^{\pi_{\theta_k}}} \bigl[\bigl(Q_g(\phi(s, a); \zeta) - g(s, a) + J_g(\theta_k) - \mathbb{E}\big[Q_g^{\pi_{\theta_k}}(s', a')\big]\bigr) \nabla_\zeta Q(\phi(s, a); \zeta_{g,0}) \bigr],
\end{align}
where the expectation for \(\nabla_\zeta \mathcal{E}(\theta_k, \zeta)\) is taken over $s' \sim P(\cdot|s, a)$, and $a' \sim \pi_{\theta_k}(\cdot|s')$. However, we cannot compute the exact gradients as we do not know the state-action occupancy measure. Thus, we perform the following gradient descent steps:
\begin{equation}
    \eta_{g, h+1}^k = \eta_{g, h}^k - c_\gamma \gamma_\xi \hat{\nabla}_\eta \mathcal{R}(\theta_k, \eta_{g, h}^k)
\end{equation}
\vspace{-.3in}

\begin{algorithm}[H]
\caption{Primal-Dual Natural Actor-Critic with Neural Critic (PDNAC-NC)}
\label{alg:pdnac_nc}
\begin{algorithmic}[1]
\REQUIRE Initial parameters $\theta_0$, $\lambda_0 = 0$; neural critic initialization $\zeta_{g,0}$ for $g \in \{r, c\}$; step sizes $\alpha, \beta, \gamma_\xi, \gamma_\omega$; outer loops $K$; inner loops $H$; truncation $T_{\max}$; projection radius $R$
\STATE Initialize: $s_0 \sim \rho(\cdot)$
\STATE \textbf{Critic Initialization:} Draw each entry of $W_l^0 \sim \mathcal{N}(0, 1)$ for $l = 1, \ldots, L$, and each entry of $b_g \sim \text{Unif}\{-1, +1\}$
\FOR{$k = 0, \ldots, K-1$}
    \STATE $\xi_{g,0}^k \gets [0, \zeta_{g,0}^\top]^\top$, $\omega_{g,0}^k \gets \mathbf{0}$ for all $g \in \{r, c\}$
    \FOR{$h = 0, \ldots, H-1$}
        \STATE $s_{kh}^0 \gets s_0$, draw $P_h^k \sim \text{Geom}(1/2)$
        \STATE $\ell_{kh} \gets (2^{P_h^k} - 1)\mathbf{1}(2^{P_h^k} \leq T_{\max}) + 1$
        \FOR{$t = 0, \ldots, \ell_{kh} - 1$}
            \STATE Take action $a_{kh}^t \sim \pi_{\theta_k}(\cdot | s_{kh}^t)$
            \STATE Observe $s_{kh}^{t+1} \sim P(\cdot | s_{kh}^t, a_{kh}^t)$
            \STATE Observe $g(s_{kh}^t, a_{kh}^t)$ for $g \in \{r, c\}$
        \ENDFOR
        \STATE $s_0 \gets s_{kh}^{\ell_{kh}}$
        \STATE Compute $v_g^{\mathrm{MLMC}}(\theta_k, \xi_{g,h}^k)$ via \eqref{eq:critic_mlmc} for $g \in \{r, c\}$
        \STATE Update $\xi_{g,h+1}^k \gets \Pi_{\mathcal{S}_R}\left(\xi_{g,h}^k - \gamma_\xi v_g^{\mathrm{MLMC}}(\theta_k, \xi_{g,h}^k)\right)$
    \ENDFOR
    \FOR{$h = 0, \ldots, H-1$}
        \STATE $s_{kh}^0 \gets s_0$, draw $Q_h^k \sim \text{Geom}(1/2)$
        \STATE $\ell_{kh} \gets (2^{Q_h^k} - 1)\mathbf{1}(2^{Q_h^k} \leq T_{\max}) + 1$
        \FOR{$t = 0, \ldots, \ell_{kh} - 1$}
            \STATE Take action $a_{kh}^t \sim \pi_{\theta_k}(\cdot | s_{kh}^t)$
            \STATE Observe $s_{kh}^{t+1} \sim P(\cdot | s_{kh}^t, a_{kh}^t)$
            \STATE Observe $g(s_{kh}^t, a_{kh}^t)$ for $g \in \{r, c\}$
        \ENDFOR
        \STATE $s_0 \gets s_{kh}^{\ell_{kh}}$
        \STATE Compute $u_g(\theta_k, \omega_{g,h}^k, \xi_g^k)$ via \eqref{eq:npg_mlmc} for $g \in \{r, c\}$
        \STATE Update $\omega_{g,h+1}^k \gets \omega_{g,h}^k - \gamma_\omega u_g(\theta_k, \omega_{g,h}^k, \xi_g^k)$
    \ENDFOR
    \STATE $\xi_g^k \gets \xi_{g,H}^k$, $\omega_g^k \gets \omega_{g,H}^k$ for $g \in \{r, c\}$
    \STATE $\omega_k \gets \omega_r^k + \lambda_k \omega_c^k$
    \STATE $\theta_{k+1} \gets \theta_k + \alpha \omega_k$
    \STATE $\lambda_{k+1} \gets \mathcal{P}_{[0, 2/\delta]}\left[\lambda_k - \beta \eta_c^k\right]$
\ENDFOR
\end{algorithmic}
\end{algorithm}
\vspace{-.2in}

\begin{equation}
    \zeta_{g, h+1}^k = \mathcal{P}_{\mathcal{S}_R} \left(\zeta_{g, h}^k - \gamma_\xi \hat{\nabla}_\zeta \mathcal{E}(\theta_k, \zeta_{g, h}^k)\right)
\end{equation}
where \(\mathcal{P}_{\mathcal{S}_R}\) is the projection onto NTK ball, \(c_\gamma\) is a scaling parameter, \(\gamma_\xi\) is the learning rate, and \(\hat{\nabla}_{\eta} \mathcal{R}\) and \(\hat{\nabla}_{\zeta} \mathcal{E}\) are estimates of the gradients. We can also represent this through the combined critic parameter \(\xi_{g, h}^k = \left[\eta_{g, h}^k\quad {\zeta_{g, h}^k}^\top\right]^\top\) and the combined gradient

\begin{align}
    v_g(\theta, \xi_{g, h}^k; z_{kh}^j) = \begin{bmatrix}
        c_\gamma \hat{\nabla}_\eta \mathcal{R}\left(\theta_k, \eta_{g, h}^k; z_{kh}^j\right) \\ \hat{\nabla}_\zeta \mathcal{E}\left(\theta_k, \zeta_{g, h}^k; z_{kh}^j\right)
    \end{bmatrix}
\end{align}
where \(\hat{\nabla}_\eta \mathcal{R}\left(\theta_k, \eta_{g, h}^k; z_{kh}^j\right)\) and \(\hat{\nabla}_\zeta \mathcal{E}\left(\theta_k, \zeta_{g, h}^k; z_{kh}^j\right)\) are single transition estimates of the corresponding gradients. To reduce bias from Markovian sampling while maintaining computational efficiency, we employ MLMC estimation. For each inner iteration $h$, we draw $P_h^k \sim \text{Geom}(1/2)$ and generate a trajectory of length $\ell_{kh} = (2^{P_h^k} - 1)\mathbf{1}(2^{P_h^k} \leq T_{\max}) + 1$. The MLMC estimate is
\begin{align}
\label{eq:critic_mlmc}
    v_g^{\mathrm{MLMC}}(\theta_k, \xi_{g,h}^k) 
    &= v_{g,kh}^0 + \begin{cases} 2^{P_h^k}\left(v_{g,kh}^{P_h^k} - v_{g,kh}^{P_h^k - 1}\right) & \text{if } 2^{P_h^k} \leq T_{\max} \\ 0, & \text{otherwise} \end{cases}
\end{align}
where $v_{g,kh}^j = 2^{-j} \sum_{t=0}^{2^j - 1} v_g(\theta_k, \xi_{g,h}^k; z_{kh}^t)$ for $j \in \{0, P_h^k - 1, P_h^k\}$. The critic parameters are updated via projected gradient descent:
\begin{align}\label{eq:critic_update}
\xi_{g,h+1}^k = \Pi_{\mathcal{S}_R}\left(\xi_{g,h}^k - \gamma_\xi v_g^{\mathrm{MLMC}}(\theta_k, \xi_{g,h}^k)\right),
\end{align}
where $\gamma_\xi$ is the critic learning rate and $\Pi_{\mathcal{S}_R}$ projects only the \(\zeta\) portion onto the set \(\mathcal{S}_R\). After $H$ inner iterations, we obtain $\xi_g^k = \xi_{g,H}^k = [\eta_g^k, (\zeta_g^k)^\top]^\top$.

The MLMC estimator achieves the same bias as averaging $T_{\max}$ samples but requires only $\mathcal{O}(\log T_{\max})$ samples on average. Moreover, since drawing from a geometric distribution does not require knowledge of the mixing time, our algorithm eliminates the mixing time oracle assumption used in prior works \citep{ganesh2025orderoptimal}.

\subsection{Natural Policy Gradient Estimation}

Given the critic estimates $\xi_g^k = [\eta_g^k, (\zeta_g^k)^\top]^\top$ from the previous subroutine, we now estimate the NPG direction $\omega_{g,\theta_k}^* = F(\theta_k)^{-1}\nabla_\theta J_g(\theta_k)$. Recall from \eqref{eq:npg_opt} that $\omega_{g,\theta_k}^*$ minimizes $f_g(\theta_k, \omega) = \frac{1}{2}\omega^\top F(\theta_k)\omega - \omega^\top \nabla_\theta J_g(\theta_k)$.

For a transition $z_{kh}^j = (s_{kh}^j, a_{kh}^j, s_{kh}^{j+1})$, we define the single-sample estimates
\begin{align}
    &\hat{F}(\theta_k; z_{kh}^j) = \nabla_\theta \log \pi_{\theta_k}(a_{kh}^j | s_{kh}^j)\!\otimes\!\nabla_\theta \log \pi_{\theta_k}(a_{kh}^j | s_{kh}^j), \label{eq:fisher_sample} \\
    &\hat{\nabla}_\theta J_g(\theta_k, \xi_g^k; z_{kh}^j) = \hat{A}_g^{\pi_{\theta_k}}(\xi_g^k; z_{kh}^j) \nabla_\theta \log \pi_{\theta_k}(a_{kh}^j | s_{kh}^j), \label{eq:pg_sample}
\end{align}
where the advantage is estimated via the neural temporal difference error
\begin{align}
    \label{eq:neural_td}
    \hat{A}_g^{\pi_{\theta_k}}(\xi_g^k; z_{kh}^j) &= g(s_{kh}^j, a_{kh}^j) - \eta_g^k + Q_g(\phi(s_{kh}^{j+1}, a'); \zeta_g^k) - Q_g(\phi(s_{kh}^j, a_{kh}^j); \zeta_g^k),
\end{align}
with $a' \sim \pi_{\theta_k}(\cdot|s_{kh}^{j+1})$. The single-sample NPG gradient estimate is
\begin{align}\label{eq:npg_grad_sample}
    \nabla_\omega \hat{f}_g(\theta_k, \omega_{g,h}^k, \xi_g^k; z_{kh}^j)
    &= \hat{F}(\theta_k; z_{kh}^j)\omega_{g,h}^k - \hat{\nabla}_\theta J_g(\theta_k, \xi_g^k; z_{kh}^j).
\end{align}

Similarly to the critic estimation, we apply MLMC to obtain variance-reduced estimates. For each inner iteration $h$, we draw $Q_h^k \sim \text{Geom}(1/2)$ and generate a trajectory of length $\ell_{kh} = (2^{Q_h^k} - 1)\mathbf{1}(2^{Q_h^k} \leq T_{\max}) + 1$. The MLMC estimate is
\begin{align}\label{eq:npg_mlmc}
    u_g(\theta_k, \omega_{g,h}^k, \xi_g^k) &= u_{g,kh}^0 + \begin{cases} 2^{Q_h^k}\left(u_{g,kh}^{Q_h^k} - u_{g,kh}^{Q_h^k - 1}\right), & \text{if } 2^{Q_h^k} \leq T_{\max} \\ 0, & \text{otherwise} \end{cases}
\end{align}
where $u_{g,kh}^j = 2^{-j} \sum_{t=0}^{2^j - 1} \nabla_\omega \hat{f}_g(\theta_k, \omega_{g,h}^k, \xi_g^k; z_{kh}^t)$ for $j \in \{0, Q_h^k - 1, Q_h^k\}$. The NPG parameters are updated as
\begin{align}\label{eq:npg_update}
\omega_{g,h+1}^k = \omega_{g,h}^k - \gamma_\omega u_g(\theta_k, \omega_{g,h}^k, \xi_g^k),
\end{align}
where $\gamma_\omega$ is the NPG learning rate. After $H$ inner iterations, we obtain $\omega_g^k = \omega_{g,H}^k$ for each $g \in \{r, c\}$.

\subsection{Key Algorithmic and Technical Novelties}

Several aspects of Algorithm~\ref{alg:pdnac_nc} merit further discussion.

\paragraph{MLMC for Mixing Time Independence.} The MLMC estimators \eqref{eq:critic_mlmc} and \eqref{eq:npg_mlmc} achieve the same bias as averaging $T_{\max}$ samples while requiring only $O(\log T_{\max})$ samples in expectation. Crucially, since the geometric distribution $\text{Geom}(1/2)$ does not depend on the mixing time, our algorithm eliminates the need for a mixing time oracle, a requirement in many prior works \citep{bai2024learning, ganesh2025orderoptimal}. This allows us to achieve convergence rates without knowledge of the exact value of $\tau_{\mathrm{mix}}$, simply an upper bound suffices.

\paragraph{Neural Tangent Kernel Regime.} The projection $\Pi_{\mathcal{S}_R}$ onto the ball of radius $R$ around the initialization $\zeta_{g,0}$ ensures the critic parameters remain in the NTK regime. In this regime, the neural network behaves approximately linearly \citep{jacot2018neural}, enabling us to leverage the linearized function class $\mathcal{F}_{R,m}$ for analysis. When the width $m$ is sufficiently large, the linearization error becomes negligible.
\paragraph{Runtime Efficiency.} Unlike recent neural actor-critic algorithms that rely on data dropping techniques to mitigate statistical dependencies \citep{gaur2024closing, ganesh2025orderoptimal} by utilizing only one sample out of every \(\tau_{\mathrm{mix}}\) steps, our MLMC-based approach eliminates the need to discard data. By correcting for Markovian bias through multi-level estimation rather than sample thinning, our algorithm utilizes the entire collected trajectory.

%% file: Sections/assumptions.tex
\section{Assumptions}\label{sec:assumptions}
\label{sec: assumptions}
We now state the assumptions required for our analysis.

\subsection{Policy Parameterization Assumptions}

\begin{assumption}[Bounded and Lipschitz Score]\label{assump:score}
For all $\theta, \theta_1, \theta_2 \in \Theta$ and $(s, a) \in \mathcal{S} \times \mathcal{A}$:
\begin{align}
    &\|\nabla_\theta \log \pi_\theta(a|s)\| \leq G_1 \\
    &\|\nabla_\theta \log \pi_{\theta_1}(a|s)\!-\!\nabla_\theta \log \pi_{\theta_2}(a|s)\|\!\leq\!G_2\|\theta_1 - \theta_2\|
\end{align}
\end{assumption}

Assumption~\ref{assump:score} requires that the score function $\nabla_\theta \log \pi_\theta(a|s)$ is uniformly bounded and Lipschitz continuous in $\theta$. This is a standard assumption in policy gradient analysis \citep{liu2020improved, agarwal2021theory, papini2018stochastic, fatkhullin2023stochastic, mondal2024improved}. It holds for many common policy parameterizations, including Gaussian policies with linearly parameterized means and bounded covariance, softmax policies with bounded features, and certain neural network policies with bounded weights.

\begin{assumption}[Fisher Non-degeneracy]\label{assump:fisher}
There exists $\mu > 0$ such that $F(\theta) \succeq \mu I_d$ for all $\theta \in \Theta$, where $I_d$ denotes the $d$-dimensional identity matrix.
\end{assumption}

Assumption~\ref{assump:fisher} ensures that the Fisher information matrix $F(\theta)$ has eigenvalues bounded away from zero. This guarantees that the NPG direction $\omega^* = F(\theta)^{-1}\nabla_\theta J(\theta)$ is well-defined and that the optimization problem \eqref{eq:npg_opt} is $\mu$-strongly convex. This assumption is standard for deriving global complexity bounds in policy gradient methods \citep{liu2020improved, zhang2021convergence, bai2023achieving, fatkhullin2023stochastic, mondal2024improved}.

\begin{assumption}[Transferred Compatible Function Approximation]\label{assump:transfer}
Define for $\theta, \omega \in \mathbb{R}^d$, $\lambda \geq 0$, and $\nu \in \Delta(\mathcal{S} \times \mathcal{A})$:
\begin{equation}
\mathcal{L}_\nu(\omega, \theta, \lambda) = \mathbb{E}_{\nu}\left[\left(\nabla_\theta \log \pi_\theta(a|s) \cdot \omega - A^{\pi_\theta}(s, a)\right)^2\right].
\end{equation}
where \(A^{\pi_\theta}(s, a) = A_r^{\pi_\theta}(s, a) + \lambda A_c^{\pi_\theta}(s, a)\) is the combined advantage function and the expectation is taken under the distribution $\nu$. Let $\omega_{\theta,\lambda}^* = \arg\min_\omega \mathcal{L}_{\nu^{\pi_\theta}}(\omega, \theta, \lambda) = F(\theta)^{-1}\nabla_\theta L(\theta, \lambda)$. There exists $\epsilon_{\mathrm{bias}} \geq 0$ such that
\begin{equation}
\mathcal{L}_{\nu^{\pi_{\theta^*}}}(\omega_{\theta,\lambda}^*, \theta, \lambda) \leq \epsilon_{\mathrm{bias}}
\end{equation}
for all $\theta \in \Theta$ and $\lambda \in [0, 2/\delta]$, where $\pi_{\theta^*}$ is optimal for the average reward CMDP \eqref{eq:cmdp_original} and $\nu^{\pi_{\theta^*}}$ is the state-action occupancy measure.
\end{assumption}

Assumption~\ref{assump:transfer} bounds the \emph{transferred} compatible function approximation error, measuring how well the NPG direction computed under the current policy $\pi_\theta$ approximates the advantage function under the optimal policy $\pi_{\theta^*}$, and is a common assumption in policy gradient literature \citep{agarwal2021theory,fatkhullin2023stochastic,ganesh2025orderoptimal,ganesh2025regret,satheesh2026regret}. The term $\epsilon_{\mathrm{bias}}$ reflects the expressivity of the policy parameterization. When the policy class is expressive enough to represent any policy, \(\epsilon_{\mathrm{bias}} = 0\). One example for where this holds is for tabular policies with softmax parameterization. If the policy class is not expressive enough, then \(\epsilon_{\mathrm{bias}} > 0\), but recent work has shown that this factor is negligble when using neural network parameterizations common in deep RL \citep{wang2020neural}.

\subsection{Neural Network Assumptions}
\begin{assumption}[Neural Critic Approximation Error]\label{assump:critic_approx}
The worst-case critic approximation error
\begin{equation}
\epsilon_{\mathrm{app}} = \sup_{\theta \in \Theta} \max_{g \in \{r,c\}} \mathbb{E}_{\nu^{\pi_\theta}}\left[\left(Q_g^{\pi_\theta}(s, a) - \Pi_{\mathcal{F}_{R,m}} Q_g^{\pi_\theta}(s, a)\right)^2\right]
\end{equation}
is finite, where $\Pi_{\mathcal{F}_{R,m}}$ denotes projection onto the linearized function class \eqref{eq:linearized_class} and the expectation is taken under the state-action occupancy measure $\nu^{\pi_\theta}$.
\end{assumption}
Assumption~\ref{assump:critic_approx} ensures that the linearized neural network class $\mathcal{F}_{R,m}$ can approximate the true Q-function $Q_g^{\pi_\theta}$ with bounded error. This is analogous to the realizability assumption in linear function approximation but adapted to the NTK regime.

%% file: Sections/theoretical_analysis.tex
\section{Theoretical Analysis}
\label{sec:theory}

In this section, we provide the theoretical guarantees for PDNAC-NC. Our analysis relies on the Neural Tangent Kernel (NTK) regime, where the neural network is well-approximated by its linearization around initialization. We proceed in three steps: (1) analyzing the neural critic by bounding its deviation from the linearized updates and establishing convergence, (2) bounding the error of the Natural Policy Gradient (NPG) estimator, and (3) establishing the global convergence of the primal-dual updates.

\subsection{Neural Critic Analysis}

We first analyze the behavior of the neural critic parameters $\xi_{g,h}^k$ (Algorithm~\ref{alg:pdnac_nc}, Line 15). To facilitate the analysis, we introduce a reference sequence of \textit{linearized updates} $\{\tilde{\xi}_{g,h}^k\}_{h \geq 0}$, defined as:
\begin{equation}
    \tilde{\xi}_{g,h+1}^k = \Pi_{S_R} \left( \tilde{\xi}_{g,h}^k - \gamma_\xi \hat{v}_g(\theta_k, \tilde{\xi}_{g,h}^k) \right), \, \tilde{\xi}_{g,0}^k = \xi_{g,0}^k = 0,
\end{equation}
where $\hat{v}_g$ is the gradient using the linearized network approximation. The following lemma bounds the deviation between the actual neural updates and this linearized sequence.

\begin{lemma}[Linearized Update Bounds]
\label{lemma:linearized_update}
Suppose the assumptions in Section~\ref{sec: assumptions} hold. Let $\gamma_\xi = \frac{8 \log T}{\lambda H}$. Then, for any epoch $k$, the expected squared difference between the neural and linearized critic iterates satisfies:
\begin{align}
    &\mathbb{E}_{\theta_k} \left[ \|\tilde{\xi}_{g,h+1}^k - \xi_{g,h+1}^k\|^2 \right] \nonumber \\
    &\leq \tilde{\mathcal{O}}\left( \frac{R^2 \tau_{\mathrm{mix}}}{\lambda T_{\max}} + \frac{R^2 \log(H/\delta) \log T_{\max}}{\lambda \sqrt{m}} \right)
\end{align}
\end{lemma}
\textit{Proof (Sketch).} This result (proven as Lemma~\ref{lemma:linearized_update} in the Appendix) leverages the projection $\Pi_{S_R}$ to keep parameters in the NTK regime, where the Hessian of the network is small. By induction, we show that the accumulated linearization error (proportional to $m^{-1/2}$) and the difference in gradients due to MLMC noise remain bounded. The projection ensures stability, preventing the neural parameters from drifting too far from the initialization where the linear approximation is valid.

Next, we establish the convergence of the critic. Let $\xi_{g,*}^k$ be the optimal parameter in the linearized function class $\mathcal{F}_{R,m}$. We provide the second-order bound for the critic estimator.

\begin{lemma}[Critic Second Order Bound]
\label{lemma:critic_mse}
Under the same settings as Lemma~\ref{lemma:linearized_update}, after $H$ inner iterations, the critic estimator $\xi_{g,H}^k$ satisfies:
\begin{align}
    &\mathbb{E}_{\theta_k} \left[ \|\xi_{g,H}^k - \xi_{g,*}^k\|^2 \right] \nonumber \\
    &\leq \tilde{\mathcal{O}} \left( \frac{\|\xi_{g,0}^k - \xi_{g,*}^k\|^2}{T^2} \!+\! \frac{c_\gamma^2 \log(H / \delta) \tau_{\mathrm{mix}}}{\lambda^2 H} \!+\! \frac{1}{\sqrt{m}} \right)
\end{align}
\end{lemma}
\textit{Proof (Sketch).} Corresponds to Lemma~\ref{lemma: critic_second_order} in the Appendix. The proof relies on analyzing the contraction of the distance to the optimum $\xi_{g,*}^k$ under the linearized dynamics. Since the objective is strongly convex in the linearized regime (Assumption 3.5), the error contracts geometrically. We then account for the variance of the stochastic MLMC gradients and the deviation from the true neural dynamics (bounded by Lemma~\ref{lemma:linearized_update}). Unrolling the recursion yields the $T^{-2}$ convergence rate for the initial error, plus steady-state error terms dominated by MLMC variance and linearization error.



\subsection{NPG Estimation Error}

Using the critic bounds, we analyze the Natural Policy Gradient estimator $\omega_g^k$. The estimator minimizes a quadratic objective whose gradient depends on the critic. We characterize its error in terms of the critic's error and MLMC sampling noise.

\begin{theorem}[NPG Estimation Error]
\label{thm:npg_estimation_error}
Suppose the assumptions of Section 3 hold. Let $\gamma_\omega = \frac{2 \log T}{\mu H}$ and $\omega_{g,0}^k = 0$. Let $\omega_{g,*}^k = F(\theta_k)^{-1} \nabla_\theta J_g(\theta_k)$. Then:
\begin{enumerate}
    \item The mean-squared error satisfies:
    \begin{align}
        \mathbb{E}_k \left[ \|\omega_g^k - \omega_{g,*}^k\|^2 \right] \leq \tilde{\mathcal{O}}\left(\frac{\tau_{\mathrm{mix}}^3}{H} + \frac{1}{\sqrt{m}} + \tau_{\mathrm{mix}}^2 \epsilon_{\mathrm{app}} \right)
    \end{align}
    \item The squared bias satisfies:
    \begin{align}
        \| \mathbb{E}_k [\omega_g^k] - \omega_{g,*}^k \|^2 \leq \tilde{\mathcal{O}}\left(\frac{\log(H / \delta)}{H} + \frac{1}{\sqrt{m}} + \tau_{\mathrm{mix}}^2\epsilon_{\mathrm{app}} \right)
    \end{align}
\end{enumerate}
\end{theorem}
\textit{Proof (Sketch).} See Theorem~\ref{thm:npg_estimation_error_bounds} in the Appendix. The NPG estimator solves a quadratic problem using stochastic gradients estimated via MLMC. We first bound the bias and variance of these gradients (Lemma~\ref{lemma: npg single sample estimate errors}), which depend on the quality of the critic $\xi_g^k$. The error in $\omega_g^k$ decomposes into optimization error (decaying with $H$), gradient variance (handled by MLMC), and gradient bias. The gradient bias is dominated by the critic's bias $\|\mathbb{E}[\xi_g^k] - \xi_{g,*}^k\|^2$ and the inherent function approximation error $\epsilon_{\mathrm{app}}$. Substituting the critic bound from Lemma~\ref{lemma:critic_mse} yields the final rates.

\subsection{Global Convergence}

Finally, we combine the NPG error bounds with the primal-dual analysis for CMDPs. The following theorem establishes the convergence rate of PDNAC-NC.

\begin{theorem}[Global Convergence]
\label{thm:main_convergence}
Suppose the assumptions in Section~\ref{sec: assumptions} hold. Let the step sizes be $\alpha = \beta = \Theta(T^{-1/4})$, $K = \Theta(T^{1/2})$, $H = \Theta(T^{1/2})$, and $T_{\max} = T$. The sequence of policies generated by Algorithm 1 satisfies:
\begin{align}
    &\frac{1}{K} \mathbb{E} \left[\sum_{k=0}^{K-1} \left( J_r^{\pi^*} - J_r(\theta_k) \right)\right] \nonumber \\
    &\leq \tilde{\mathcal{O}}\left(\sqrt{\epsilon_{\mathrm{bias}}} + \sqrt{\epsilon_{\mathrm{app}}} + T^{-1/4} + m^{-1/4}\right) \\
    &\frac{1}{K} \mathbb{E} \left[\sum_{k=0}^{K-1} - J_c(\theta_k) \right] \nonumber \\
    &\leq \tilde{\mathcal{O}}\left(\sqrt{\epsilon_{\mathrm{bias}}} + \sqrt{\epsilon_{\mathrm{app}}} + T^{-1/4} + m^{-1/4}\right)
\end{align}
\end{theorem}
\textit{Proof (Sketch).} See Theorem~\ref{thm: cmdp convergence} in the Appendix. The proof follows the standard primal-dual analysis for CMDPs but replaces the exact NPG update with our estimated $\omega_g^k$. We bound the drift terms using the bias and variance of $\omega_g^k$. The approximation errors $\epsilon_{\mathrm{bias}}$ and $\epsilon_{\mathrm{app}}$ appear as irreducible terms due to the limitations of the policy and critic function classes. The $m^{-1/4}$ term reflects the NTK linearization error, and taking \(m = \Theta(T)\) yields an $\epsilon$-optimal solution.

%% file: Sections/conclusion.tex
\section{Conclusion}
\label{sec: conclusion}
In this work, we provide the first global convergence guarantees for infinite-horizon average-reward Constrained MDPs (CMDPs) utilizing multi-layer neural network critics and general policy parameterizations. By leveraging Neural Tangent Kernel (NTK) theory, we successfully characterize the function approximation error of deep neural critics in the continuous state-action regime. Furthermore, by integrating a nested Multi-Level Monte Carlo (MLMC) estimation subroutine, our algorithm bypasses the need for sample-thinning techniques and mixing-time oracles, ensuring that the entirety of the collected Markovian trajectories is utilized. We establish a convergence rate of $\tilde{\mathcal{O}}(T^{-1/4})$ for both the optimality gap and constraint violation.

While our work establishes strong theoretical guarantees for neural actor-critic methods in CMDPs, the achieved rate of $\tilde{\mathcal{O}}(T^{-1/4})$ is not order-optimal when compared to recent analysis for unconstrained natural actor-critic methods under Markovian sampling. The primary technical bottleneck in the constrained setting arises due to the mean squared error of the critic, as prior work in the linear regime uses the sharper squared bias term \citep{ganesh2025orderoptimal, xu2025global}. However, directly applying the squared bias bound in our case is difficult due to the projection operator from the NTK analysis. Overcoming this technical challenge to improve the bounds and achieve optimal rates remains an interesting open problem.

%% file: Sections/Appendix/related_work.tex
\section{Related Work}
\label{appendix: related work}
\paragraph{Constrained Markov Decision Processes (CMDPs) and Safe RL} 
The CMDP framework is the standard mathematical formulation for safe reinforcement learning, originally formalized by \citet{altman2021constrained}. Early solutions relied on linear programming, which struggled with the curse of dimensionality. Recently, primal-dual Policy Gradient (PG) and Natural Policy Gradient (NPG) methods have become the standard for high-dimensional CMDPs \citep{ding2020natural, paternain2019constrained}. While initial non-asymptotic analyses of these methods were limited to tabular settings or linear function approximation \citep{xu2021crpo, chen2022learning, bai2023achieving}, recent works have extended them to general policy parameterizations. However, most of these guarantees are confined to the discounted reward setting.

\textbf{Average-Reward Setting:} 
While discounted reward MDPs have been extensively studied, infinite-horizon average-reward MDPs are notoriously more challenging due to the lack of contraction properties in their Bellman operators. Early theoretical advancements in average-reward CMDPs were restricted to tabular cases or required generative models. \citet{chen2022learning,agarwal2022regret,agarwal2022concave} provided convergence rates for average-reward CMDPs but relied on  tabular policies. More recently, \citet{bai2024learning,xu2025global} explored general parameterized policies for average-reward CMDPs via a primal-dual policy gradient algorithm, but their analysis relies heavily on access to exact gradients or limits the critic's function approximation capacity, leaving the gap for deep neural network critics unexplored.

\textbf{Actor-Critic Methods and Neural Function Approximation:} 
Actor-Critic (AC) algorithms are ubiquitous in practical RL, but their theoretical analysis has historically lagged behind their empirical success. Foundational analyses of AC algorithms established finite-time global convergence but primarily assumed linear critic approximations \citep{khodadadian2021finite, cayci2021linear}. To bridge the gap between theory and deep RL practice, recent literature has adopted Neural Tangent Kernel (NTK) theory to analyze multi-layer neural network critics \citep{jacot2018neural}. Works such as \citet{fu2020single, cayci2024finite} and \citet{tian2023convergence} integrated neural critics into AC methods but analyzed them strictly within standard, unconstrained MDPs. Even the most recent state-of-the-art analyses of neural actor-critic methods such as \citet{gaur2024closing} and \citet{ganesh2025orderoptimal} achieve sample complexities of $\tilde{\mathcal{O}}(\epsilon^{-3})$ and $\tilde{\mathcal{O}}(\epsilon^{-2})$ respectively, but are limited to the discounted, unconstrained setting. Our work fundamentally extends this line of inquiry by applying NTK analysis to the coupled primal-dual dynamics required for the constrained average-reward setting.

\textbf{Markovian Sampling and Multi-Level Monte Carlo (MLMC):} 
A primary bottleneck in the analysis of RL algorithms under Markovian sampling is the statistical dependence across sequential transitions. Traditional analyses mitigate this by utilizing data-dropping or sample-thinning techniques, where only one out of every $\tau_{\text{mix}}$ samples is used to construct gradient estimates \citep{gaur2024closing, ganesh2025orderoptimal}. This approach not only wastes collected data but also necessitates the restrictive assumption of possessing a mixing-time oracle to know the exact value of the mixing time. To circumvent this, recent theoretical works have adapted Multi-Level Monte Carlo (MLMC) estimation techniques \citep{blanchet2015unbiased} to RL \citep{beznosikov2023first}. By drawing trajectory lengths from a geometric distribution, MLMC provides unbiased gradient estimates that systematically correct for Markovian bias without requiring sample thinning. While \citet{xu2025global} successfully applied MLMC to primal-dual natural policy gradients, their work was restricted to linear critics. Our algorithm is the first to integrate MLMC to simultaneously evaluate a neural network critic and estimate the NPG direction, thereby entirely removing the need for a mixing-time oracle in neural actor-critic methods.

%% file: Sections/Appendix/mlmc_estimation.tex
\section{Preliminary Results for MLMC Estimation}
\label{appendix: mlmc estimation}

We begin by establishing fundamental properties of the Multi-Level Monte Carlo (MLMC) estimator that will be used throughout the analysis.

\subsection{Markovian Sampling Bounds}

The following lemma bounds the variance of sample averages under Markovian sampling.

\begin{lemma}[Lemma 1, \citet{beznosikov2023first}]
\label{lemma: markov_var}
Consider a time-homogeneous, ergodic Markov chain $\{Z_t\}_{t \geq 0}$ with unique stationary distribution $d_Z$ and mixing time $\tau_{\mathrm{mix}}$. Let $h(Z)$ be a function satisfying $\|h(Z_t) - \mathbb{E}_{d_Z}[h(Z)]\|^2 \leq \sigma^2$ for all $t \geq 0$. Then
\begin{equation}
\mathbb{E}\left[\left\|\frac{1}{N}\sum_{t=0}^{N-1} h(Z_t) - \mathbb{E}_{d_Z}[h(Z)]\right\|^2\right] \leq \frac{C_1 \tau_{\mathrm{mix}}}{N} \sigma^2,
\end{equation}
where $C_1 = 16(1 + 1/\ln^2 4)$ and the expectation is over trajectories from any initial distribution.
\end{lemma}

\subsection{MLMC Estimator Properties}

The MLMC estimator combines samples at multiple resolutions to achieve low bias with few samples.

\begin{lemma}[Lemma B.2, \citet{xu2025global}]
\label{lemma: mlmc_properties}
Consider a time-homogeneous, ergodic Markov chain $\{Z_t\}_{t \geq 0}$ with stationary distribution $d_Z$ and mixing time $\tau_{\mathrm{mix}}$. Let $g(\cdot)$ satisfy:
\begin{itemize}
    \item $\|\mathbb{E}_{d_Z}[g(Z)] - \nabla F(x)\|^2 \leq \delta^2$ (bias at stationarity)
    \item $\|g(Z_t) - \mathbb{E}_{d_Z}[g(Z)]\|^2 \leq \sigma^2$ for all $t \geq 0$ (bounded variance)
\end{itemize}
Let $Q \sim \mathrm{Geom}(1/2)$ and define the MLMC estimator
\begin{equation}\label{eq:mlmc_def}
g^{\mathrm{MLMC}} = g^0 + \begin{cases} 2^Q(g^Q - g^{Q-1}), & \text{if } 2^Q \leq T_{\max} \\ 0, & \text{otherwise} \end{cases}
\end{equation}
where $g^j = 2^{-j} \sum_{t=0}^{2^j - 1} g(Z_t)$ for $j \in \{0, Q-1, Q\}$. Then:
\begin{align}
    &\mathbb{E}[g^{\mathrm{MLMC}}] = \mathbb{E}[g^{\lfloor \log_2 T_{\max} \rfloor}] \\
    &\mathbb{E}[\|g^{\mathrm{MLMC}} - \nabla F(x)\|^2] \leq O(\sigma^2 \tau_{\mathrm{mix}} \log^2 T_{\max} + \delta^2) \\
    &\|\mathbb{E}[g^{\mathrm{MLMC}}] - \nabla F(x)\|^2 \leq O(\sigma^2 \tau_{\mathrm{mix}} T_{\max}^{-1} + \delta^2)
\end{align}
\end{lemma}

%% file: Sections/Appendix/neural_critic_analysis.tex
\section{Neural Critic Analysis}
\label{appendix: neural critic analysis}
We analyze the convergence of the neural critic estimator. Recall from Section~\ref{subsec: neural critic estimation} that the combined critic parameter is $\xi_g^k$. 
We define the linearized single-sample gradient of the critic objective at $\xi_{g,0}$ as
\begin{align}
    &\hat{v}_g(\theta_k, \xi_{g, h}^k; z_{kh}^j) = \mathbf{A}_g(\theta_k; z_{kh}^j) \xi_{g, h}^k - \mathbf{b}_g(\theta_k; z_{kh}^j),
\end{align}
where the single-sample matrix and vector are defined as:
\begin{align}
&\mathbf{A}_g(\theta_k; z_{kh}^j) = \begin{bmatrix} c_\gamma & 0 \\ \psi_g^j & \psi_g^j \left(\psi_g^j - \psi_g^{j+1}\right)^\top \end{bmatrix}, \label{eq:A_def} \\
&\mathbf{b}_g(\theta_k; z_{kh}^j)\!=\!\begin{bmatrix} c_\gamma g(s_{kh}^j, a_{kh}^j) \\ \left(g(s_{kh}^j, a_{kh}^j) - Q_{g, 0}^j + Q_{g, 0}^{j+1}\right) \psi_g^j \end{bmatrix}\!, \label{eq:b_def}
\end{align}
where $\psi_g^j := \nabla_\zeta Q_g(\phi(s_{kh}^j, a_{kh}^j); \zeta_{g,0})$, $\psi_g^{j+1} := \nabla_\zeta Q_g(\phi(s_{kh}^{j+1}, a'); \zeta_{g,0})$, \(Q_{g,0}^j = Q(\phi(s_{kh}^j, a_{kh}^j); \zeta_{g, 0})\), \(Q_{g,0}^{j+1} = Q(\phi(s_{kh}^{j+1}, a'); \zeta_{g, 0})\), and we use the zero-initialization for \(\zeta_{g, 0}\). We define the expected population matrices as:
\begin{align}
    \mathbf{A}_g(\theta_k) &= \mathbb{E}_{(s, a)\sim\nu_g^{\pi_{\theta_k}}, s' \sim P(s, a), a' \sim \pi_{\theta_k}(\cdot \mid s')} \left[\mathbf{A}_g(\theta_k; z)\right], \\
    \mathbf{b}_g(\theta_k) &= \mathbb{E}_{(s, a)\sim\nu_g^{\pi_{\theta_k}}, s' \sim P(s, a), a' \sim \pi_{\theta_k}(\cdot \mid s')} \left[\mathbf{b}_g(\theta_k; z)\right].
\end{align}
where \(z\) is the transition. For the following sections, we denote the above expectation as \(\mathbb{E}_{\theta_k}\).

\begin{lemma}[NTK Critic Bounds]
\label{lemma: ntk critic bounds}
Fix an outer iteration index $k$. Let \(\zeta_{g,h}^k \in \mathcal{S}_R \quad \forall h \in \{0,1,\dots,H\}\) and \(R = O(\log T)\). Then, for all inner steps $h \in \{1, \dots, H\}$, there exist positive constants $C_1, C_1', \{C_i\}_{i=2}^5$ such that with probability at least $1 - \delta - 2L \exp(-Cm)$, the following statements hold:
\begin{align}
    &\|\nabla_\zeta Q(\phi(s, a); \zeta_{g,h}^k)\| \le C_1, \quad |Q(\phi(s, a); \zeta_{g,h}^k)| \le C_1' \sqrt{\log(H/\delta)} \label{eqn: grad Q and Q bound}\\
    &\|v_g(\theta_k, \zeta_{g,h}^k; z_{kh}^j) - \hat v_g(\theta_k, \zeta_{g,h}^k; z_{kh}^j)\| \le C_2 Rm^{-1/2} \sqrt{\log(H/\delta)} \label{eqn: linearization error bound} \\
    &\langle v_g(\theta_k, \zeta_{g,h}^k; z_{kh}^j) - \hat v_g(\theta_k, \zeta_{g,h}^k; z_{kh}^j), \zeta_{g,h}^k - \zeta_g^*\rangle \le C_3 Rm^{-1/2} \sqrt{\log(H/\delta)} \label{eqn: linearization inner product error bound}\\
    &|\hat{Q}(\phi(s, a); \zeta_{g,h}^k) - Q(\phi(s, a); \zeta_{g,h}^k)| \le C_4 m^{-1/2} \sqrt{\log(H/\delta)} \label{eqn: Q function linearization error} \\
    &\|\nabla_\zeta Q(\phi(s, a); \zeta_0) - \nabla_\zeta Q(\phi(s, a); \zeta_{g,h}^k)\| \le C_5 m^{-1/2} \sqrt{\log(H/\delta)} \label{eqn: grad Q function linearization error}
\end{align}
\end{lemma}
\begin{proof}
    Equation~\eqref{eqn: grad Q and Q bound} follows from Lemma~D.2 and~D.3 in \citet{ke2024improved}, equations~\eqref{eqn: linearization error bound} and~\eqref{eqn: linearization inner product error bound} follow from Lemma~D.5, and equations~\eqref{eqn: Q function linearization error} and~\eqref{eqn: grad Q function linearization error} follow from Lemma~D.4.
\end{proof}

\begin{lemma}[Bounds on Single-Sample Linearized Critic Matrices]
\label{lemma: single sample bounds}
Fix an outer iteration index $k$ and an inner step $h \in \{0, \dots, H\}$. Let $z = (s, a, s')$ denote a single transition with $a' \sim \pi_{\theta_k}(\cdot \mid s')$. Then, under the NTK bounds in Lemma~\ref{lemma: ntk critic bounds}, we have:
\begin{align}
    &\|\mathbf{A}_g(\theta_k; z)\| \le c_\gamma + 3 C_1^2, \\
    &\|\mathbf{b}_g(\theta_k; z)\| \le c_\gamma + 3C_1C_1^{'}\sqrt{\log(H / \delta)}, \\
    &\|\mathbf{A}_g(\theta_k; z) - \mathbf{A}_g(\theta_k)\|^2 \le \mathcal{O}\left(c_\gamma^2 + C_1^4\right), \\
    &\|\mathbf{b}_g(\theta_k; z) - \mathbf{b}_g(\theta_k)\|^2 \le \mathcal{O}\left(c_\gamma^2 + C_1^2 C_1^{'2} \log(H / \delta)\right)
\end{align}
\end{lemma}

\begin{proof}
    The bound for $\mathbf{A}_g$ follows from the triangle inequality:
    \begin{align}
        \|\mathbf{A}_g(\theta_k; z)\| &\leq c_\gamma\!+\! \|\nabla_\zeta Q_g(\phi(s, a); \zeta_{g, 0})\| \!+\! \|\nabla_\zeta Q_g(\phi(s, a); \zeta_{g, 0})\|\|\nabla_\zeta Q_g(\phi(s, a); \zeta_{g, 0}) \!-\! \nabla_\zeta Q_g(\phi(s', a'); \zeta_{g, 0})\| \\
        &\leq c_\gamma + C_1 + C_1 (2C_1) \\
        &\leq c_\gamma + 3 C_1^2
    \end{align}
    Similarly for $\mathbf{b}_g$, utilizing the fact that $|g(s,a)| \le 1$:
    \begin{align}
        \|\mathbf{b}_g(\theta_k; z)\| &\leq c_\gamma |g(s, a)| + |g(s, a) - Q(\phi(s, a); \zeta_{g, 0}) + Q(\phi(s', a'); \zeta_{g, 0})| \|\nabla_\zeta Q(\phi(s, a); \zeta_{g, 0})\| \\
        &\leq c_\gamma + \left(1 + 2C_1^{'}\sqrt{\log(H / \delta)}\right)C_1 \\
        &\leq c_\gamma + 3C_1C_1^{'}\sqrt{\log(H / \delta)}
    \end{align}
    For the squared bias bounds, we have
    \begin{align}
        \|\mathbf{A}_{g}(\theta_k; z) - \mathbf{A}_g(\theta_k)\|^2 &\leq 2\|\mathbf{A}_g(\theta_k; z)\|^2 + 2\|\mathbf{A}_g(\theta_k)\|^2 \\
        &\leq 2\|\mathbf{A}_g(\theta_k; z)\|^2 + \mathbb{E}_{\theta_k}\left[2\|\mathbf{A}_g(\theta_k; z)\|^2\right] \\ 
        &\leq \mathcal{O}\left(c_\gamma^2 + C_1^4\right)
    \end{align}
    and
    \begin{align}
        \|\mathbf{b}_{g}(\theta_k; z) - \mathbf{b}_g(\theta_k)\|^2 &\leq 2\|\mathbf{b}_g(\theta_k; z)\|^2 + 2\|\mathbf{b}_g(\theta_k)\|^2 \\
        &\leq 2\|\mathbf{b}_g(\theta_k; z)\|^2 + 2\mathbb{E}_{\theta_k}[\|\mathbf{b}_g(\theta_k; z)\|^2] \\
        &\leq \mathcal{O}\left(c_\gamma^2 + C_1^2 C_1^{'2} \log(H / \delta)\right)
    \end{align}
\end{proof}

\begin{lemma}[PSD of $\mathbf{A}_g(\theta_k)$]
    \label{lemma: psd a_g}
    Fix \(k\) and let Assumption~\ref{assump:ergodic} hold. Then if \(\xi_g \in \mathrm{ker}(\mathbf{A}_g(\theta_k))^\perp\), there exists a \(\lambda > 0\) such that the following holds:
    \begin{align}
        \xi_g^\top \mathbf{A}_g(\theta_k) \xi_g \geq \frac{\lambda}{2} \|\xi_g\|_2^2
    \end{align}
\end{lemma}

\begin{proof}
    We first observe \(\mathrm{ker}(\mathbf{A}_g(\theta_k))\). By definition, for any \(\xi_g \in \mathrm{ker}(\mathbf{A}_g(\theta_k))\), we have
    \begin{align}
        \mathbf{A}_g(\theta_k) \xi_g = \mathbb{E}_{\theta_k} \begin{bmatrix}
            c_\gamma & 0 \\
            \psi_{g} & \psi_g\left(\psi_g - \psi_g^{'}\right)^\top
        \end{bmatrix} \xi_g = \mathbf{0} \implies c_\gamma \eta_g = 0 \implies \eta_g = 0
    \end{align}
    This yields
    \begin{align}
        \mathbf{A}_{g}^\zeta(\theta_k) \zeta_g := \mathbb{E}_{\theta_k} \left[\psi_g\left(\psi_g - \psi_g^{'}\right)^\top \zeta_g\right] = \mathbf{0}
    \end{align}
    We then define
    \begin{align}
        \mathcal{S}_e = \left\{\zeta_g \in \mathbb{R}^{m(n+(L-1)m)} \mid \psi_g(s, a)^\top \zeta_g = c,\, \forall (s, a) \in \mathrm{supp}(\nu_g^{\pi_{\theta_k}}), c \in \mathbb{R}\right\}
    \end{align}
    We wish to show that \(\mathrm{ker}\left(\mathbf{A}_g(\theta_k)\right) = \mathcal{S}_e\). First, we choose an arbitrary \(\zeta_g \in \mathcal{S}_e\). Then
    \begin{align}
        \mathbf{A}_{g}^\zeta(\theta_k) \zeta_g &= \mathbb{E}_{\theta_k} \left[\psi_g\left(\psi_g^\top \zeta_g - {\psi_g^{'}}^\top \zeta_g\right)\right] \\
        &\stackrel{(a)}{=} \mathbb{E}_{\theta_k} \left[\psi_g (0)\right] \\
        &= 0
    \end{align}
    where \((a)\) holds by the definition of \(\mathcal{S}_e\). Thus, \(\zeta_g \in \mathrm{ker}(\mathbf{A}_g^\zeta(\theta_k))\) and \(\mathcal{S}_e \subseteq \mathrm{ker}(\mathbf{A}_g^\zeta(\theta_k))\). Similarly, we choose an arbitrary \(\zeta'_g \in \mathrm{ker}(\mathbf{A}_g^\zeta(\theta_k))\). By definition, this yields
    \begin{align}
        \mathbf{A}_g^\zeta(\theta_k) \zeta'_g = \mathbf{0} \implies  \zeta'^\top_g \mathbf{A}_g^\zeta(\theta_k) \zeta'_g = 0
    \end{align}
    Expanding out this quadratic form yields
    \begin{align}
        \zeta'^\top_g \mathbf{A}_g^\zeta(\theta_k) \zeta'_g = \zeta'^\top_g \mathbb{E}_{\theta_k} \left[\psi_g\left(\psi_g - \psi_g^{'}\right)^\top\right] \zeta'_g = \mathbb{E}_{\theta_k} \left[\left(\zeta'^\top_g \psi_g\right)\left(\psi_g^\top \zeta'_g - {\psi_g^{'\top}} \zeta'_g \right)\right] = \mathcal{E}(\psi_g^\top \zeta'_g, \psi_g^\top \zeta'_g) = 0
    \end{align}
    where \(\mathcal{E}(\psi_g^\top \zeta'_g, \psi_g^\top \zeta'_g)\) is the Dirichlet form. Since the Dirichlet form is 0 and our MDP is irreducible and aperiodic by ergodicity, \(\psi_g^\top \zeta'_g\) is constant on the support of \(\nu_g^{\pi_{\theta_k}}\). Thus, \(\zeta'_g \in \mathcal{S}_e\) and \(\mathrm{ker}(\mathbf{A}_g^\zeta(\theta_k)) \subseteq \mathcal{S}_e\). Since both \(\mathcal{S}_e \subseteq \mathrm{ker}(\mathbf{A}_g^\zeta(\theta_k))\) and \(\mathrm{ker}(\mathbf{A}_g^\zeta(\theta_k)) \subseteq \mathcal{S}_e\), \(\mathrm{ker}(\mathbf{A}_g^\zeta(\theta_k)) = \mathcal{S}_e\) and \(\mathrm{ker}(\mathbf{A}_g^\zeta(\theta_k))^\perp = \mathcal{S}_e^\perp\). 

    We have that the Dirichlet form is related to the variance through the spectral gap of the MDP \(\lambda^*\):
    \begin{align}
        \mathcal{E}\left(\psi_g^\top \zeta_g, \psi_g^\top \zeta_g\right) \geq \lambda^* \mathrm{Var}\left(\psi_g^\top \zeta_g\right)
    \end{align}
    Since we are given \(\xi_g \in \mathrm{ker}\left(\mathbf{A}_g(\theta_k)\right)^\perp\), \(\zeta_g \in \mathcal{S}_e^\perp\), and \(\mathrm{Var}(\psi_g^\top \zeta_g) > 0\) on the support of \(\nu_{g}^{\pi_{\theta_k}}\). We define
    \begin{align}
        \lambda = \min_{\|\zeta_g\|^2_2 = 1, \zeta_g \in \mathcal{S}_e^\perp} \zeta_g^\top \mathbf{A}_g^\zeta(\theta_k) \zeta_g > 0 
    \end{align}
    which implies that \(\zeta_g^\top \mathbf{A}_g^\zeta(\theta_k) \zeta_g > \lambda \|\zeta_g\|_2^2\). Finally, we have
    \begin{align}
        \xi_g^\top \mathbf{A}_g(\theta_k) \xi_g &= c_\gamma \eta_g^2 + \eta_g \zeta_g^\top \mathbb{E}_{\theta_k} \left[\psi_g\right] + \zeta_g^\top \mathbf{A}_g^\zeta(\theta_k) \zeta_g \\
        &\stackrel{(a)}{\geq} c_\gamma \eta_g^2 - C_1|\eta_g| |\zeta_g| + \lambda \|\zeta_g\|_2^2 \\
        &\stackrel{(b)}{\geq} \frac{\lambda}{2} \|\xi_g\|_2^2
    \end{align}
    where \((a)\) uses Lemma~\ref{lemma: ntk critic bounds}a, and \((b)\) holds when \(c_\gamma \geq \frac{\lambda}{2} + \frac{C_1^2}{2 \lambda}\).
\end{proof}

\begin{lemma}[MLMC Linearized Critic Matrix Error Bounds]
\label{lemma: mlmc linearized critic bounds}
Fix an outer iteration index $k$ and inner iteration $h \in \{0, \dots, H\}$. Let 
\(\mathbf{A}_{g,kh}^{\mathrm{MLMC}}\) and \(\mathbf{b}_{g,kh}^{\mathrm{MLMC}}\) denote the MLMC estimates of the linearized critic matrices at step $h$. Under the NTK bounds in Lemma~\ref{lemma: ntk critic bounds} and the single-sample bounds in Lemma~\ref{lemma: single sample bounds}, we have
\begin{align}
    \left\|\mathbb{E}_{k, h}\left[\mathbf{A}_{g, kh}^{\mathrm{MLMC}}\right] - \mathbf{A}_g(\theta_k)\right\|^2 
    &= \mathcal{O}\Big(\left(c_\gamma^2 + C_1^4\right) \tau_\mathrm{mix} T_{\max}^{-1}\Big), \\
    \mathbb{E}_{k, h}\left[\left\|\mathbf{A}_{g, kh}^{\mathrm{MLMC}} - \mathbf{A}_g(\theta_k)\right\|^2\right] 
    &= \mathcal{O}\Big(\left(c_\gamma^2 + C_1^4\right) \tau_\mathrm{mix} \log T_{\max}\Big), \\
    \left\|\mathbb{E}_{k, h}\left[\mathbf{b}_{g, kh}^{\mathrm{MLMC}}\right] - \mathbf{b}_g(\theta_k)\right\|^2 
    &= \mathcal{O}\Big(\left(c_\gamma^2 + C_1^2 C_1'^2 \log(H/\delta)\right) \tau_\mathrm{mix} T_{\max}^{-1}\Big), \\
    \mathbb{E}_{k, h}\left[\left\|\mathbf{b}_{g, kh}^{\mathrm{MLMC}} - \mathbf{b}_g(\theta_k)\right\|^2\right] 
    &= \mathcal{O}\Big(\left(c_\gamma^2 + C_1^2 C_1'^2 \log(H/\delta)\right) \tau_\mathrm{mix} \log T_{\max}\Big),
\end{align}
where $T_{\max}$ is the maximum rollout length and \(\mathbb{E}_{k, h}\) denotes the conditional expectation at timestep \(k\) given the history of the current inner loop.
\end{lemma}

\begin{proof}
    To find the MLMC error bounds, we first find the values of \(\sigma_\mathbf{A}^2\), \(\sigma_\mathbf{b}^2\), \(\delta_{\mathbf{A}}^2\), \(\delta_{\mathbf{b}}^2\), which denote the squared bias and variance terms for the single sample estimates. Since by definition \(\mathbb{E}_{\theta_k}\left[\mathbf{A}_g(\theta_k; z)\right] = \mathbf{A}_g(\theta_k)\) and \(\mathbb{E}_{\theta_k}\left[\mathbf{b}_g(\theta_k; z)\right] = \mathbf{b}_g(\theta_k)\), \(\delta_\mathbf{A}^2 = \delta_\mathbf{b}^2 = 0\). Additionally, from Lemma~\ref{lemma: single sample bounds}, we have
    \begin{align}
        \mathbb{E}_{\theta_k} \left[\|\mathbf{A}_g(\theta_k; z) - \mathbf{A}_g(\theta_k)\|^2\right] &\leq \mathcal{O}\left(c_\gamma^2 + C_1^4\right) \\
        \mathbb{E}_{\theta_k}\left[\|\mathbf{b}_g(\theta_k; z) - \mathbf{b}_g(\theta_k)\|^2\right] &\leq \mathcal{O}\left(c_\gamma^2 + C_1^2 C_1'^2 \log(H / \delta)\right) 
    \end{align}
    Thus, by Lemma~\ref{lemma: mlmc_properties}, we have
    \begin{align}
        \left\|\mathbb{E}_{k, h}\left[\mathbf{A}_{g, kh}^{\mathrm{MLMC}}\right] - \mathbf{A}_g(\theta_k)\right\|^2 &=
        \mathcal{O}\left(\sigma_\mathbf{A}^2 \tau_{\mathrm{mix}} T_{\max}^{-1} + \delta_{\mathbf{A}}^2\right) = \mathcal{O}\Big(\left(c_\gamma^2 + C_1^4\right) \tau_\mathrm{mix} T_{\max}^{-1}\Big), \\
        \mathbb{E}_{k, h}\left[\left\|\mathbf{A}_{g, kh}^{\mathrm{MLMC}} - \mathbf{A}_g(\theta_k)\right\|^2\right] &=
        \mathcal{O}\left(\sigma_\mathbf{A}^2 \tau_{\mathrm{mix}} \log T_{\max} + \delta_{\mathbf{A}}^2\right)         = \mathcal{O}\Big(\left(c_\gamma^2 + C_1^4\right) \tau_\mathrm{mix} \log T_{\max}\Big), \\
        \left\|\mathbb{E}_{k, h}\left[\mathbf{b}_{g, kh}^{\mathrm{MLMC}}\right] - \mathbf{b}_g(\theta_k)\right\|^2 &=
        \mathcal{O}\left(\sigma_\mathbf{b}^2 \tau_{\mathrm{mix}} T_{\max}^{-1} + \delta_{\mathbf{b}}^2\right)
        = \mathcal{O}\Big(\left(c_\gamma^2 + C_1^2 C_1'^2 \log(H/\delta)\right) \tau_\mathrm{mix} T_{\max}^{-1}\Big), \\
        \mathbb{E}_{k, h}\left[\left\|\mathbf{b}_{g, kh}^{\mathrm{MLMC}} - \mathbf{b}_g(\theta_k)\right\|^2\right] &=
        \mathcal{O}\left(\sigma_\mathbf{b}^2 \tau_{\mathrm{mix}} \log T_{\max} + \delta_{\mathbf{b}}^2\right)
        = \mathcal{O}\Big(\left(c_\gamma^2 + C_1^2 C_1'^2 \log(H/\delta)\right) \tau_\mathrm{mix} \log T_{\max}\Big),
    \end{align}
\end{proof}

We also provide the following bounds on the original MLMC gradient estimate.
\begin{lemma}[MLMC Critic Parameter Error Bounds]
    \label{lemma: mlmc original critic bounds}
    Assume the same setting as Lemma~\ref{lemma: mlmc linearized critic bounds}. Let \(v_g(\theta_k, \xi_{g, h}^k) = \mathbb{E}_{\theta_k} \left[v_g(\theta_k, \xi_{g, h}^k; z)\right]\). Then we have the following bounds:
    \begin{align}
        \|\mathbb{E}_{k, h}\left[v_g^{\mathrm{MLMC}}(\theta_k, \xi_{g, h}^k)\right] - v_g(\theta_k, \xi_{g, h}^k)\|^2 &\leq \mathcal{O}\Big(\left(c_\gamma^2 + C_1^2 C_1'^2 \log(H/\delta)\right) \tau_\mathrm{mix} T_{\max}^{-1}\Big) \\
        \mathbb{E}_{k, h}\left\|v_g^{\mathrm{MLMC}}(\theta_k, \xi_{g, h}^k) - v_g(\theta_k, \xi_{g, h}^k)\right\|^2 &\leq \mathcal{O}\Big(\left(c_\gamma^2 + C_1^2 C_1'^2 \log(H/\delta)\right) \tau_\mathrm{mix} \log T_{\max}\Big)
    \end{align}
\end{lemma}

\begin{proof}
    By the definition of the single sample combined gradient estimate, we have
    \begin{align}
        \|v_g(\theta_k, \xi_{g, h}^k; z)\| &\leq c_\gamma\|\hat{\nabla}_\eta \mathcal{R}(\theta_k, \eta_{g, h}^k; z)\| + \|\hat{\nabla}_\zeta \mathcal{E}(\theta_k, \zeta_{g, h}^k; z)\| \\
        &= c_\gamma \|\eta_{g, h}^k - g(s_{kh}^j, a_{kh}^j)\| \nonumber \\
        &\quad + \|Q_g(\phi(s, a); \zeta_{g, h}^k) + \eta_{g, h}^k - g(s, a) - Q_g(\phi(s', a'); \zeta_{g, h}^k)\| \| \nabla_\zeta Q_g(\phi(s, a); \zeta_{g, h}^k)\| \\
        &\leq \mathcal{O}\left(c_\gamma + C_1C_1' \sqrt{\log (H / \delta)}\right)
    \end{align}
    Additionally, we have
    \begin{align}
        \sigma_{v}^2 = \|v_g(\theta_k, \xi_{g, h}^k; z) - v_g(\theta_k, \xi_{g, h}^k)\|^2 &\leq 2\|v_g(\theta_k, \xi_{g, h}^k; z)\|^2 + 2\|v_g(\theta_k, \xi_{g, h}^k)\|^2 \\
        &\leq 2\|v_g(\theta_k, \xi_{g, h}^k; z)\|^2  + \mathbb{E}_{\theta_k} \left[2\|v_g(\theta_k, \xi_{g, h}^k)\|^2\right] \\
        &\leq \mathcal{O}\left(c_\gamma^2 + C_1^2 C_1'^2 \log(H / \delta)\right)
    \end{align}
    and
    \begin{align}
        \delta_v^2 = \|\mathbb{E}_{\theta_k}\left[v_g(\theta_k, \xi_{g, h}^k; z)\right] - v_g(\theta_k, \xi_{g, h}^k)\| = 0
    \end{align}
    Then by Lemma~\ref{lemma: mlmc_properties}, we have
    \begin{align*}
        \|\mathbb{E}_{k, h}\left[v_g^{\mathrm{MLMC}}(\theta_k, \xi_{g, h}^k)\right] - v_g(\theta_k, \xi_{g, h}^k)\|^2 &\leq \mathcal{O}\left(\sigma_v^2 \tau_{\mathrm{mix}} T_{\max}^{-1} + \delta_v^2\right) = \mathcal{O}\Big(\left(c_\gamma^2 + C_1^2 C_1'^2 \log(H/\delta)\right) \tau_\mathrm{mix} T_{\max}^{-1}\Big)\\
        \mathbb{E}_{k, h}\left\|v_g^{\mathrm{MLMC}}(\theta_k, \xi_{g, h}^k) - v_g(\theta_k, \xi_{g, h}^k)\right\|^2 &\leq \mathcal{O}\left(\sigma_v^2 \tau_{\mathrm{mix}} \log T_{\max} + \delta_v^2\right) = \mathcal{O}\Big(\left(c_\gamma^2 + C_1^2 C_1'^2 \log(H/\delta)\right) \tau_\mathrm{mix} \log T_{\max}\Big)
    \end{align*}
\end{proof}

\begin{lemma}[First Order Linearization Error of the MLMC Critic Estimator]
\label{lemma: first order linearization error mlmc estimator}
Let the MLMC estimator use truncation level \(T_{\max}\). Then
\[
    \mathbb{E}_{k,h} \left[\left\|v_g^{\mathrm{MLMC}}(\theta_k, \xi_{g,h}^k) - \hat{v}_g^{\mathrm{MLMC}}(\theta_k, \xi_{g, h}^k) \right\|\right] = \mathcal{O}\left(C_2 R m^{-1/2}\sqrt{\log(H/\delta)} \log T_{\max}\right).
\]
\end{lemma}

\begin{proof}
    We already have from Lemma~\ref{lemma: ntk critic bounds} that \(\|v_g(\theta_k, \xi_{g, h}^k; z_{kh}^j) - \hat{v}_g(\theta_k, \xi_{g, h}^k; z_{kh}^j) \| \leq C_2 R m^{-1/2}\sqrt{\log(H / \delta)}\) from equation~\eqref{eqn: linearization error bound}. Then from the definition of the MLMC estimator, we have
    \begin{align}
        &\mathbb{E}_{k, h} \left[\|v_g^{\mathrm{MLMC}}(\theta_k, \xi_{g, h}^k) - \hat{v}_g^{\mathrm{MLMC}}(\theta_k, \xi_{g, h}^k)\|\right]\\
        &\quad = \mathbb{E}_{k, h} \left[ \mathbb{E}_{P_h^k} \left[\left\|v_{g, kh}^0 - \hat{v}_{g, kh}^0 + \begin{cases}
            2^{P_h^k} \left(v_{g, kh}^{P_h^k} - v_{g, kh}^{P_h^k - 1} - \hat{v}_{g, kh}^{P_h^k} + \hat{v}_{g, kh}^{P_h^k - 1}\right) & \text{if } 2^{P_h^k} \leq T_\mathrm{max} \\ 0 & \text{otherwise}
        \end{cases} \right\| \Bigg| k, h\right]\right]
    \end{align}
    By definition, we have
    \begin{align}
        \left\|v_{g, kh}^{j} - \hat{v}_{g, kh}^j\right\| &= \left\|\frac{1}{2^j} \sum_{t = 0}^{2^j - 1} \left(v_{g}(\theta_k, \xi_{g, h}^k; z_{kh}^t) - \hat{v}_{g}(\theta_k, \xi_{g, h}^k; z_{kh}^t)\right)\right\| \\
        &\leq \frac{1}{2^j} \sum_{t = 0}^{2^j - 1} \left\|v_{g}(\theta_k, \xi_{g, h}^k; z_{kh}^t) - \hat{v}_{g}(\theta_k, \xi_{g, h}^k; z_{kh}^t)\right\| \\
        &\stackrel{(a)}{\leq} \frac{1}{2^j} \sum_{t = 0}^{2^j - 1} C_2 R m^{-1/2} \sqrt{\log(H / \delta)} \\
        &= C_2 R m^{-1/2} \sqrt{\log(H / \delta)}
    \end{align}
    where \((a)\) uses equation~\eqref{eqn: linearization error bound}. Thus, we have
    \begin{align}
        &\mathbb{E}_{k, h} \left[\|v_g^{\mathrm{MLMC}}(\theta_k, \xi_{g, h}^k) - \hat{v}_g^{\mathrm{MLMC}}(\theta_k, \xi_{g, h}^k)\|\right] \\
        &\quad \leq \mathbb{E}_{k, h} \left[ \mathbb{E}_{P_h^k} \left[C_2 R m^{-1/2} \sqrt{\log(H / \delta)} + 2^{P_h^k+1} \left(C_2 R m^{-1/2} \sqrt{\log(H / \delta)} \right) \mathbb{I}\left(2^{P_h^k} \leq T_{\mathrm{max}}\right) \Bigg| k, h\right]\right] \\
        &\quad \stackrel{(a)}{\leq} \mathbb{E}_{k, h} \left[C_2 R m^{-1/2} \sqrt{\log(H / \delta)} + \sum_{p=0}^{\lfloor \log_2 T_\mathrm{max} \rfloor} 2^p \left(2 C_2 R m^{-1/2} \sqrt{\log(H / \delta)}\right) \frac{1}{2}\frac{1}{2^{p-1}}\right] \\
        &\quad = \mathbb{E}_{k, h} \left[C_2 R m^{-1/2} \sqrt{\log(H / \delta)} + \sum_{p=0}^{\lfloor \log_2 T_\mathrm{max} \rfloor} \left(2 C_2 R m^{-1/2} \sqrt{\log(H / \delta)}\right)\right] \\
        &\quad = \mathcal{O}\left(\log_2 (T_\mathrm{max}) C_2 R m^{-1/2} \sqrt{\log(H / \delta)}\right)
    \end{align}
    where \((a)\) uses the definition of the expectation for the truncated geometric distribution.
\end{proof}

\begin{lemma}[Second Order Linearization Error of the MLMC Critic Estimator]
    \label{lemma: second order linearization error mlmc estimator}
    Assume the same setting as Lemma~\ref{lemma: first order linearization error mlmc estimator}. Then
    \begin{align}
        \mathbb{E}_{k, h} \left[\|v_g^{\mathrm{MLMC}}(\theta_k, \xi_{g, h}^k) - \hat{v}_g^{\mathrm{MLMC}}(\theta_k, \xi_{g, h}^k)\|^2\right] \leq \mathcal{O}\left(C_2^2 R^2 m^{-1} \log(H / \delta) T_{\max} \right)
    \end{align}
\end{lemma}

\begin{proof}
    We have
    \begin{align}
        &\mathbb{E}_{k, h} \left[\|v_g^{\mathrm{MLMC}}(\theta_k, \xi_{g, h}^k) - \hat{v}_g^{\mathrm{MLMC}}(\theta_k, \xi_{g, h}^k)\|^2\right]\\
        &= \mathbb{E}_{k, h} \left[ \mathbb{E}_{P_h^k} \left[\left\|v_{g, kh}^0 - \hat{v}_{g, kh}^0 + 2^{P_h^k} \left(v_{g, kh}^{P_h^k} - v_{g, kh}^{P_h^k - 1} - \hat{v}_{g, kh}^{P_h^k} + \hat{v}_{g, kh}^{P_h^k - 1}\right) \mathbb{I}\left(2^{P_h^k} \leq T_{\max}\right)\right\|^2 \Bigg| k, h\right]\right] \\
        &\leq 2\mathbb{E}_{k, h} \left[\|v_{g, kh}^0 - \hat{v}_{g, kh}^0\|^2\right] \\
        &\quad + 2\mathbb{E}_{k, h} \left[\mathbb{E}_{P_h^k} \left[\left\|  2^{P_h^k} \left(v_{g, kh}^{P_h^k} - v_{g, kh}^{P_h^k - 1} - \hat{v}_{g, kh}^{P_h^k} + \hat{v}_{g, kh}^{P_h^k - 1}\right) \mathbb{I}\left(2^{P_h^k} \leq T_{\max}\right) \right\|^2 \Bigg| k, h\right]\right] \\
        &\stackrel{(a)}{\leq} 2C_2^2 R^2 m^{-1} \log(H / \delta) + 4C_2^2 R^2 m^{-1} \log(H / \delta) \mathbb{E}_{k, h} \left[\mathbb{E}_{P_h^k} \left[2^{2P_h^k} \mathbb{I}\left(2^{P_h^k} \leq T_{\max}\right)\Bigg| k, h\right] \right] \\
        &= 2C_2^2 R^2 m^{-1} \log(H / \delta) + 4C_2^2 R^2 m^{-1} \log(H / \delta) \sum_{p=0}^{\lfloor \log_2 T_{\max} \rfloor} 2^{2p} \frac{1}{2^{p+1}} \\
        &= \mathcal{O}\left(C_2^2 R^2 m^{-1} \log(H / \delta) T_{\max} \right)
    \end{align}
    where \((a)\) follows from Lemma~\ref{lemma: first order linearization error mlmc estimator}
\end{proof}

Now that we have bounds for the estimation error for the MLMC gradient and its parameters, we can find the bound between the linearized and original updates.

\begin{lemma}[Linearized Update Bounds]
    \label{lemma: linearized update bounds}
    We define the linearized sequence \(\{\tilde{\xi}_{g, h}^k\}_{h \geq 0}\) that represents the linearized iterates of the neural update \(\hat{v}_g\):
    \begin{align}
        \tilde{\xi}_{g, h+1}^k = \Pi_{\mathcal{S}_R}\left(\tilde{\xi}_{g, h}^k - \gamma_\xi \hat{v}_g(\theta_k, \tilde{\xi}_{g, h}^k)\right) \quad \tilde{\xi}_{g, 0}^k = \xi_{g, 0}^k = 0
    \end{align}
    Let \(\Pi_\perp\) be the orthogonal projection onto \(\mathrm{ker}\left(\mathbf{A}_g(\theta_k)\right)\). Then the following result holds:
    \begin{align}
        \mathbb{E}_{\theta_k} \left\|\tilde{\xi}_{g, h+1}^k - \xi_{g, h+1}^k\right\|^2 \leq \mathcal{O}\left(\frac{R^2(c_\gamma^2 + C_1^4) \tau_{\mathrm{mix}}\log T}{\lambda T_{\max}} + \frac{R^2 \sqrt{\log (H / \delta)} \log T_{\max} \log T}{\lambda m^{1/2}} + \frac{R^2 T_{\max} \log(H / \delta) \log T}{H\lambda m}\right)
    \end{align}
\end{lemma}

\begin{proof}
    Using the non-expansiveness of \(\Pi_{\mathcal{S}_R}\) and Lemma~\ref{lemma: ntk critic bounds}, we have
    \begin{align*}
        &\mathbb{E}_{k, h} \left\|\tilde{\xi}_{g, h+1}^k - \xi_{g, h+1}^k\right\|^2 = \mathbb{E}_{k,h} \left\|\Pi_{\mathcal{S}_R}\left(\xi_{g, h}^k - \gamma_\xi v_g^{\mathrm{MLMC}}(\theta_k, \xi_{g, h}^k)\right) - \Pi_{\mathcal{S}_R}\left(\tilde{\xi}_{g, h}^k - \gamma_\xi \hat{v}_g^{\mathrm{MLMC}}\left(\theta_k, \tilde{\xi}_{g, h}^k\right)\right)\right\|^2 \\
        &\leq \mathbb{E}_{k,h} \left\|\xi_{g, h}^k - \gamma_\xi v_g^{\mathrm{MLMC}}(\theta_k, \xi_{g, h}^k) -\left( \tilde{\xi}_{g, h}^k - \gamma_\xi \hat{v}_g^{\mathrm{MLMC}}\left(\theta_k, \tilde{\xi}_{g, h}^k\right)\right)\right\|^2 \\
        &\leq \mathbb{E}_{k,h} \Bigg\|\left(\xi_{g, h}^k - \tilde{\xi}_{g, h}^k\right) - \gamma_\xi \left(\mathbf{A}_{g, kh}^\mathrm{MLMC}(\theta_k) \Big(\xi_{g, h}^k - \tilde{\xi}_{g, h}^k\right) + v_g^{\mathrm{MLMC}}(\theta_k, \xi_{g, h}^k) - \hat{v}_g^{\mathrm{MLMC}}(\theta_k, \xi_{g, h}^k) \Big)\Bigg\|^2 \\
        &\leq \mathbb{E}_{k, h} \left\|\xi_{g, h}^k - \tilde{\xi}_{g, h}^k\right\|^2 -2\gamma_\xi \mathbb{E}_{k, h}\left\langle \xi_{g, h}^k - \tilde{\xi}_{g, h}^k, \mathbf{A}_{g, kh}^\mathrm{MLMC}(\theta_k) \left(\xi_{g, h}^k - \tilde{\xi}_{g, h}^k\right) \right\rangle \\
        &\quad -2\gamma_\xi \mathbb{E}_{k, h}\left\langle \xi_{g, h}^k - \tilde{\xi}_{g, h}^k,  v_g^\mathrm{MLMC}(\theta_k, \xi_{g, h}^k) - \hat{v}_g^\mathrm{MLMC}(\theta_k, \xi_{g, h}^k)\right\rangle + \mathcal{O}\left(\gamma_\xi^2 C_2^2 R^2 m^{-1} \log(H / \delta) T_{\max}\right)  \\
        &\quad + 2\gamma_\xi^2 \mathbb{E}_{k, h}\left\|\mathbf{A}_{g, kh}^\mathrm{MLMC}(\theta_k) \left(\xi_{g, h}^k - \tilde{\xi}_{g, h}^k\right)\right\|^2\\
        &\leq \mathbb{E}_{k, h} \left\|\xi_{g, h}^k - \tilde{\xi}_{g, h}^k\right\|^2 -2\gamma_\xi \mathbb{E}_{k, h}\left\langle \xi_{g, h}^k - \tilde{\xi}_{g, h}^k, \mathbf{A}_{g, kh}^\mathrm{MLMC}(\theta_k) \left(\xi_{g, h}^k - \tilde{\xi}_{g, h}^k\right) \right\rangle \\
        &\quad + \mathcal{O}\left(\gamma_\xi R^2 \log T_{\max} m^{-1/2} \sqrt{\log (H / \delta)} \right) + \mathcal{O}\left(\gamma_\xi^2 C_2^2 R^2 m^{-1} \log(H / \delta) T_{\max}\right) \\
        &\quad + 4\gamma_\xi^2 \mathbb{E}_{k, h}\left\|\mathbf{A}_g(\theta_k) \left(\xi_{g, h}^k - \tilde{\xi}_{g, h}^k\right)\right\|^2 + 4\gamma_\xi^2 \mathbb{E}_{k, h}\left\|\mathbf{A}_{g, kh}^{\mathrm{MLMC}}(\theta_k) \left(\xi_{g, h}^k - \tilde{\xi}_{g, h}^k\right) - \mathbf{A}_g(\theta_k) \left(\xi_{g, h}^k - \tilde{\xi}_{g, h}^k\right)\right\|^2 \\
        &\leq \mathbb{E}_{k, h} \left\|\xi_{g, h}^k - \tilde{\xi}_{g, h}^k\right\|^2 - 2\gamma_\xi \mathbb{E}_{k, h}\left\langle \Pi_\perp\left(\xi_{g, h}^k - \tilde{\xi}_{g, h}^k\right), \mathbf{A}_{g}(\theta_k) \Pi_{\perp}\left(\xi_{g, h}^k - \tilde{\xi}_{g, h}^k\right) \right\rangle + \mathcal{O}\left(\frac{\gamma_\xi R^2 (c_\gamma^2 + C_1^4) \tau_{\mathrm{mix}}}{T_{\max}}\right)  \\
         &\quad + \mathcal{O}\left(\gamma_\xi R^2 \log T_{\max} m^{-1/2} \sqrt{\log (H / \delta)} \right) + \mathcal{O}\left(\gamma_\xi^2 C_2^2 R^2 m^{-1} \log(H / \delta) T_{\max}\right) \\
         &\quad + \mathcal{O}\left(\gamma_\xi^2 \left(c_\gamma^2 + C_1^4\right) \right) \mathbb{E}_{k, h} \left\|\Pi_{\perp} \left(\xi_{g, h}^k - \tilde{\xi}_{g, h}^k\right)\right\|^2 + \mathcal{O}\left(\gamma_\xi^2 \left(c_\gamma^2 + C_1^4\right) \tau_{\mathrm{mix}} \log T_{\max}\right) \left\|\Pi_{\perp} \left(\xi_{g, h}^k - \tilde{\xi}_{g, h}^k\right)\right\|^2 \\
         &\leq \mathbb{E}_{k, h} \left\|\xi_{g, h}^k - \tilde{\xi}_{g, h}^k\right\|^2 + \mathcal{O}\left(\frac{\gamma_\xi R^2 (c_\gamma^2 + C_1^4) \tau_{\mathrm{mix}}}{T_{\max}}\right) + \mathcal{O}\left(\gamma_\xi R^2 \log T_{\max} m^{-1/2} \sqrt{\log (H / \delta)} \right) \\
         &\quad + \mathcal{O}\left(\gamma_\xi^2 C_2^2 R^2 m^{-1} \log(H / \delta) T_{\max}\right) \\
         &\leq \mathcal{O}\left(\frac{(h+1)\gamma_\xi R^2 (c_\gamma^2 + C_1^4) \tau_{\mathrm{mix}}}{T_{\max}}\right) + \mathcal{O}\left((h+1)\gamma_\xi R^2  m^{-1/2} \sqrt{\log (H / \delta)}\log T_{\max} \right) \\
         &\quad + \mathcal{O}\left((h+1)\gamma_\xi^2 C_2^2 R^2 m^{-1} \log(H / \delta) T_{\max}\right)
    \end{align*}
    If we substitute \(\gamma_\xi = \frac{8 \log T}{\lambda H}\) and take expectations on both sides, we have
    \begin{align*}
        \mathbb{E}_{\theta_k} \left\|\tilde{\xi}_{g, h+1}^k - \xi_{g, h+1}^k\right\|^2 \leq \mathcal{O}\left(\frac{R^2(c_\gamma^2 + C_1^4) \tau_{\mathrm{mix}}\log T}{\lambda T_{\max}} + \frac{R^2 \sqrt{\log (H / \delta)} \log T_{\max} \log T}{\lambda m^{1/2}} + \frac{R^2 T_{\max} \log(H / \delta) \log T}{H\lambda m}\right)
    \end{align*}
\end{proof}

Before proceeding further with the neural critic analysis, we provide the following Lemma from \citet{ganesh2025orderoptimal}, adopted to the CMDP setting.

\begin{lemma}[Lemma 9, \citet{ganesh2025orderoptimal}]
Let $Z_k := \{z \in \mathbb{R}^d : z = \mathbf{A}_g(\theta_k)^\dagger \mathbf{b}_g(\theta_k) + v, \, v \in \ker(\mathbf{A}_g(\theta_k))\}$ denote the set of minimum-norm least-squares solutions to $\mathbf{A}_g(\theta_k) z \approx \mathbf{b}_g(\theta_k)$, and let $\xi_{g, *}^k$ be the projection of a fixed point $\xi_0$ onto $Z_k$. Then, under the update rule
\[
    \tilde{\xi}_{g, h}^k = \Pi_{\mathcal{S}_R} \left( \tilde{\xi}_{g, h-1}^k - \gamma_\xi \hat{v}_g(\theta_k; \tilde{\xi}_{g, h-1}^k) \right),
\]
where $\hat{v}_g(\theta_k; \tilde{\xi}_{g, h}^k) \in \ker(\mathbf{A}_g(\theta_k))^\perp$, it holds that
\[
    \tilde{\xi}_{g, h}^k - \xi_{g, *}^k \in \ker(\mathbf{A}_g(\theta_k))^\perp, \quad \text{for all } h \geq 0.
\]
\end{lemma}

Next, we derive the second order bound.

\begin{lemma}[Critic Second Order Bound]
\label{lemma: critic_second_order}
Assume the same setting as Lemma~\ref{lemma: linearized update bounds}. Then, after $H$ inner critic updates, the critic iterate $\xi_{g,H}^k$ satisfies
\begin{align}
    \mathbb{E}_{\theta_k}\!\left[\|\xi_{g,H}^k - \xi_{g,*}^k\|^2\right] &\leq \mathcal{O}\Bigg(\frac{\mathbb{E}_{\theta_k}\|\tilde{\xi}_{g, 0}^k - \xi_{g, *}^k\|^2}{T^2} + \frac{c_\gamma^4 \tau_{\mathrm{mix}} \log(H / \delta) R^2\log T}{\lambda^4 T_{\max}} + \frac{R^2 \sqrt{\log(H / \delta)} \log T \log T_{\max}}{\lambda m^{-1/2}} \nonumber \\
    &\quad  + \frac{R^2 T_{\max} \log (H / \delta) \log T}{H \lambda m} + \frac{c_\gamma^2 \log(H / \delta) \tau_{\mathrm{mix}} \log T_{\max} \log T}{\lambda^2 H}\Bigg)
\end{align}
\end{lemma}

\begin{proof}
    \begin{align*}
        &\mathbb{E}_{k, h}\left\|\tilde{\xi}_{g, h+1}^k - \xi_{g, *}^k\right\|^2 = \mathbb{E}_{k, h}\|\Pi_{\mathcal{S}_R}\left(\tilde{\xi}_{g, h}^k - \gamma_\xi v_g^{\mathrm{MLMC}}(\theta_k; \tilde{\xi}_{g, h}^k)\right) - \Pi_{\mathcal{S}_R}(\xi_{g, *}^k)\|^2 \\
        &\leq \mathbb{E}_{k, h}\|\tilde{\xi}_{g, h}^k  - \gamma_\xi v_g^{\mathrm{MLMC}}(\theta_k; \tilde{\xi}_{g, h}^k) - \xi_{g, *}^k\|^2 \\
        &= \mathbb{E}_{k, h}\|\tilde{\xi}_{g, h}^k - \xi_{g, *}^k\|^2 - 2\gamma_\xi \mathbb{E}_{k, h}\left\langle\tilde{\xi}_{g, h}^k - \xi_{g, *}^k, v_g^{\mathrm{MLMC}}(\theta_k; \tilde{\xi}_{g, h}^k)\right\rangle + \gamma_\xi^2 \mathbb{E}_{k, h}\|v_g^{\mathrm{MLMC}}(\theta_k; \tilde{\xi}_{g, h}^k)\|^2 \\
        &= \mathbb{E}_{k, h}\|\tilde{\xi}_{g, h}^k - \xi_{g, *}^k\|^2 - 2\gamma_\xi \mathbb{E}_{k, h}\left\langle\tilde{\xi}_{g, h}^k - \xi_{g, *}^k, \hat{v}_g^{\mathrm{MLMC}}(\theta_k; \tilde{\xi}_{g, h}^k)\right\rangle + \gamma_\xi^2 \mathbb{E}_{k, h}\|v_g^{\mathrm{MLMC}}(\theta_k; \tilde{\xi}_{g, h}^k)\|^2 \nonumber \\
        &\quad  - 2\gamma_\xi \mathbb{E}_{k, h}\left\langle\tilde{\xi}_{g, h}^k - \xi_{g, *}^k, v_g^{\mathrm{MLMC}}(\theta_k; \tilde{\xi}_{g, h}^k) - \hat{v}_g^{\mathrm{MLMC}}(\theta_k; \tilde{\xi}_{g, h}^k)\right\rangle \\
        &\leq \|\tilde{\xi}_{g, h}^k - \xi_{g, *}^k\|^2 - 2\gamma_\xi\left\langle\tilde{\xi}_{g, h}^k - \xi_{g, *}^k, \hat{v}_g^{\mathrm{MLMC}}(\theta_k; \tilde{\xi}_{g, h}^k)\right\rangle - 2\gamma_\xi\left\langle\tilde{\xi}_{g, h}^k - \xi_{g, *}^k, v_g^{\mathrm{MLMC}}(\theta_k; \tilde{\xi}_{g, h}^k) - \hat{v}_g^{\mathrm{MLMC}}(\theta_k; \tilde{\xi}_{g, h}^k)\right\rangle \nonumber \\
        &\quad + 2\gamma_\xi^2 \mathbb{E}_{k, h}\|v_g(\theta_k; \tilde{\xi}_{g, h}^k)\|^2 + 2\gamma_\xi^2 \mathbb{E}_{k, h}\left\|v_g^{\mathrm{MLMC}}(\theta_k; \tilde{\xi}_{g, h}^k) - v_g(\theta_k; \tilde{\xi}_{g, h}^k)\right\|^2\\
        &\leq \|\tilde{\xi}_{g, h}^k - \xi_{g, *}^k\|^2 - 2\gamma_\xi\left\langle\tilde{\xi}_{g, h}^k - \xi_{g, *}^k, \mathbf{A}_g(\theta_k) (\tilde{\xi}_{g, h}^k - \xi_{g, *}^k)\right\rangle + \mathcal{O}\left( \gamma_\xi^2 \left(c_\gamma^2 + C_1'^2 C_1^2\right) \log(H / \delta) \tau_{\mathrm{mix}} \log T_{\max}\right) \nonumber \\
        &\quad - 2\gamma_\xi\left\langle\tilde{\xi}_{g, h}^k - \xi_{g, *}^k, \hat{v}_g^{\mathrm{MLMC}}(\theta_k; \tilde{\xi}_{g, h}^k) - \mathbf{A}_g(\theta_k) (\tilde{\xi}_{g, h}^k - \xi_{g, *}^k)\right\rangle + \mathcal{O}\left(\gamma_\xi \log (T_\mathrm{max}) C_2 R^2 m^{-1/2} \sqrt{\log(H / \delta)}\right)\\
        &\stackrel{(a)}{\leq} \|\tilde{\xi}_{g, h}^k - \xi_{g, *}^k\|^2 - \gamma_\xi \lambda \|\tilde{\xi}_{g, h}^k - \xi_{g, *}^k\| - 2\gamma_\xi\left\langle\tilde{\xi}_{g, h}^k - \xi_{g, *}^k, v_g(\theta_k; \tilde{\xi}_{g, h}^k) - \mathbf{A}_g(\theta_k) (\tilde{\xi}_{g, h}^k - \xi_{g, *}^k)\right\rangle \nonumber \\
        &\quad + \mathcal{O}\left(\gamma_\xi \log (T_\mathrm{max}) C_2 R^2 m^{-1/2} \sqrt{\log(H / \delta)}\right) + \mathcal{O}\left( \gamma_\xi^2 \left(c_\gamma^2 + C_1'^2 C_1^2\right) \log(H / \delta) \tau_{\mathrm{mix}} \log T_{\max}\right)
    \end{align*}
    where we use Lemma~\ref{lemma: ntk critic bounds} to bound \(v_g(\theta_k, \xi_{g, h}^k)\) and \((a)\) follows from Lemma~\ref{lemma: psd a_g}. We take the conditional expectation \(\mathbb{E}_{k, h}\) 
    \begin{align}
        &\mathbb{E}_{k, h}\left\|\tilde{\xi}_{g, h+1}^k - \xi_{g, *}^k\right\|^2 \leq (1 - \gamma_\xi \lambda) \|\tilde{\xi}_{g, h}^k - \xi_{g, *}^k\|^2 - 2\gamma_\xi\left\langle\tilde{\xi}_{g, h}^k - \xi_{g, *}^k, \mathbb{E}_{k, h} \left[\hat{v}_g^{\mathrm{MLMC}}(\theta_k; \tilde{\xi}_{g, h}^k) - \mathbf{A}_g(\theta_k) (\tilde{\xi}_{g, h}^k - \xi_{g, *}^k)\right]\right\rangle \nonumber \\
        &\quad + \mathcal{O}\left(\gamma_\xi\log (T_\mathrm{max}) C_2 R^2 m^{-1/2} \sqrt{\log(H / \delta)}\right) + \mathcal{O}\left( \gamma_\xi^2 \left(c_\gamma^2 + C_1'^2 C_1^2\right) \log(H / \delta) \tau_{\mathrm{mix}} \log T_{\max}\right)
    \end{align}
    We bound the second term as 
    \begin{align}
        & - 2\gamma_\xi\left\langle\tilde{\xi}_{g, h}^k - \xi_{g, *}^k, \mathbb{E}_{k, h} \left[\hat{v}_g^{\mathrm{MLMC}}(\theta_k; \tilde{\xi}_{g, h}^k) - \mathbf{A}_g(\theta_k) (\tilde{\xi}_{g, h}^k - \xi_{g, *}^k)\right]\right\rangle \\
        &\leq \frac{\gamma_\xi \lambda}{2} \|\tilde{\xi}_{g, h}^k - \xi_{g, *}^k\|^2 + \frac{2 \gamma_\xi}{\lambda} \left\|\mathbb{E}_{k, h} \left[\hat{v}_g^{\mathrm{MLMC}}(\theta_k; \tilde{\xi}_{g, h}^k) - \mathbf{A}_g(\theta_k)\left(\tilde{\xi}_{g, h}^k - \xi_{g, *}^k\right)\right]\right\|^2 \\
        &\leq \frac{\gamma_\xi \lambda}{2} \|\tilde{\xi}_{g, h}^k - \xi_{g, *}^k\|^2 + \frac{2 \gamma_\xi}{\lambda} \left\|\left(\mathbb{E}_{k, h} \left[\mathbf{A}_{g, kh}^{\mathrm{MLMC}}\right] - \mathbf{A}_g(\theta_k)\right) \tilde{\xi}_{g, h}^k - \left(\mathbf{b}_g(\theta_k) - \mathbb{E}_{k, h}\left[\mathbf{b}_{g, kh}^\mathrm{MLMC}\right]\right)\right\|^2 \\
        &\leq \frac{\gamma_\xi \lambda}{2} \|\tilde{\xi}_{g, h}^k - \xi_{g, *}^k\|^2 + \frac{4 \gamma_\xi \Delta_{\mathbf{A}}^2 \|\tilde{\xi}_{g, h}^k\|^2 + 4 \gamma_\xi \Delta_\mathbf{b}^2}{\lambda} \\
        &\leq \frac{\gamma_\xi \lambda}{2} \|\tilde{\xi}_{g, h}^k - \xi_{g, *}^k\|^2 + \frac{8 \gamma_\xi \Delta^2_{\mathbf{A}} \left(\|\tilde{\xi}_{g, h}^k - \xi_{g, *}^k\| + 4\lambda^{-2} \Lambda_\mathbf{b}^2 \right) + 4 \gamma_\xi \Delta_b^2}{\lambda}
    \end{align}
    where \(\Delta_{\mathbf{A}}^2 = \left\|\mathbb{E}_{k, h}\left[\mathbf{A}_{g, kh}^{\mathrm{MLMC}}\right] - \mathbf{A}_g(\theta_k)\right\|^2\), \(\Delta_\mathbf{b}^2 = \left\|\mathbb{E}_{k, h}\left[\mathbf{b}_{g, kh}^{\mathrm{MLMC}}\right] - \mathbf{b}_g(\theta_k)\right\|^2\), and \(\Lambda_{\mathbf{b}}^2 = \|\mathbf{b}_{g, kh}^{\mathrm{MLMC}}\|^2\). The last inequality is due to \(\|\xi_{*}^k\|^2 = \|\mathbf{A}_g(\theta_k)^{-1}\mathbf{b}_g(\theta_k)\|^2 \leq \frac{4\Lambda_{\mathbf{b}}^2}{\lambda^2}\). We substitute this back to yield
    \begin{align}
        &\mathbb{E}_{k, h}\left\|\tilde{\xi}_{g, h+1}^k - \xi_{g, *}^k\right\|^2 \leq \left(1 - \frac{\gamma_\xi \lambda}{2} + \frac{8 \gamma_\xi \Delta_{\mathbf{A}}^2}{\lambda}\right) \|\tilde{\xi}_{g, h}^k - \xi_{g, *}^k\|^2 + \frac{4 \gamma_\xi}{\lambda} \left(8 \lambda^{-2} \Lambda_{\mathbf{b}}^2 \Delta_{\mathbf{A}}^2 + \Delta_{\mathbf{b}}^2\right) \\
        &\quad + \mathcal{O}\left(\gamma_\xi \log (T_\mathrm{max}) C_2 R^2 m^{-1/2} \sqrt{\log(H / \delta)}\right) + \mathcal{O}\left( \gamma_\xi^2 \left(c_\gamma^2 + C_1'^2 C_1^2\right) \log(H / \delta) \tau_{\mathrm{mix}} \log T_{\max}\right)
    \end{align}
    For \(\Delta_{\mathbf{A}}^2 \leq \frac{\lambda^2}{32}\), we have
    \begin{align}
        &\mathbb{E}_{k, h}\left\|\tilde{\xi}_{g, h+1}^k - \xi_{g, *}^k\right\|^2 \leq \left(1 - \frac{\gamma_\xi \lambda}{4}\right) \|\tilde{\xi}_{g, h}^k - \xi_{g, *}^k\|^2 + \frac{4 \gamma_\xi}{\lambda} \left(8 \lambda^{-2} \Lambda_{\mathbf{b}}^2 \Delta_{\mathbf{A}}^2 + \Delta_{\mathbf{b}}^2\right) \\
        &\quad + \mathcal{O}\left(\gamma_\xi \log (T_\mathrm{max}) C_2 R^2 m^{-1/2} \sqrt{\log(H / \delta)}\right) + \mathcal{O}\left( \gamma_\xi^2 \left(c_\gamma^2 + C_1'^2 C_1^2\right) \log(H / \delta) \tau_{\mathrm{mix}} \log T_{\max}\right)
    \end{align}
    Taking expectations on both sides with respect to \(\theta_k\) and substituting \(\gamma_\xi = \frac{8 \log T}{\lambda H}\) yields
    \begin{align*}
        &\mathbb{E}_{\theta_k}\left\|\tilde{\xi}_{g, H}^k - \xi_{g, *}^k\right\|^2 \leq \left(1 - \frac{\gamma_\xi \lambda}{4}\right)^H \mathbb{E}_{\theta_k}\|\tilde{\xi}_{g, 0}^k - \xi_{g, *}^k\|^2 + \frac{4}{\gamma_\xi \lambda} \Bigg(\frac{4 \gamma_\xi}{\lambda} \left(8 \lambda^{-2} \Lambda_{\mathbf{b}}^2 \Delta_{\mathbf{A}}^2 + \Delta_{\mathbf{b}}^2\right) \\
        &\quad + \mathcal{O}\left( \gamma_\xi\log (T_\mathrm{max}) C_2 R^2 m^{-1/2} \sqrt{\log(H / \delta)}\right) + \mathcal{O}\left( \gamma_\xi^2 \left(c_\gamma^2 + C_1'^2 C_1^2\right) \log(H / \delta) \tau_{\mathrm{mix}} \log T_{\max}\right) \Bigg) \\
        &\leq \exp\left(-\frac{H\gamma_\xi \lambda}{4}\right) \mathbb{E}_{\theta_k}\|\tilde{\xi}_{g, 0}^k - \xi_{g, *}^k\|^2 + \frac{16}{\lambda^2} \left(8 \lambda^{-2} \Lambda_{\mathbf{b}}^2 \Delta_{\mathbf{A}}^2 + \Delta_{\mathbf{b}}^2\right) \\
        &\quad + \mathcal{O}\left(\lambda^{-1}\log (T_\mathrm{max}) C_2 R^2 m^{-1/2} \sqrt{\log(H / \delta)}\right) + \mathcal{O}\left(\lambda^{-1} \gamma_\xi \left(c_\gamma^2 + C_1'^2 C_1^2\right) \log(H / \delta) \tau_{\mathrm{mix}} \log T_{\max}\right) \Bigg) \\
        &\leq \mathcal{O}\left(\frac{\mathbb{E}_{\theta_k}\|\tilde{\xi}_{g, 0}^k - \xi_{g, *}^k\|^2}{T^2}\right) + \mathcal{O}\left(\lambda^{-4} (c_\gamma^2 + C_1^2 C_1'^2) \log (H / \delta) \left(c_\gamma^2 + C_1^4\right) \tau_{\mathrm{mix}} T_{\max}^{-1}\right) \\
        &\quad + \mathcal{O}\left(\lambda^{-1}\log (T_\mathrm{max}) C_2 R^2 m^{-1/2} \sqrt{\log(H / \delta)}\right) + \mathcal{O}\left(\frac{\log T \left(c_\gamma^2 + C_1'^2 C_1^2\right) \log(H / \delta) \tau_{\mathrm{mix}} \log T_{\max}}{\lambda^2 H}\right)
    \end{align*}
    Combining this with Lemma~\ref{lemma: linearized update bounds} yields
    \begin{align}
        &\mathbb{E}_{\theta_k}\left[\|\xi_{g, H}^k - \xi_{g, *}^k\|^2\right] \leq 2\mathbb{E}_{\theta_k}\left[\|\tilde{\xi}_{g, H}^k - \xi_{g, *}^k\|^2\right] + 2\mathbb{E}_{\theta_k}\left[\|\tilde{\xi}_{g, H}^k - \xi_{g, H}^k\|^2\right] \\
        &\leq \mathcal{O}\Bigg(\frac{\mathbb{E}_{\theta_k}\|\tilde{\xi}_{g, 0}^k - \xi_{g, *}^k\|^2}{T^2} + \frac{c_\gamma^4 \tau_{\mathrm{mix}} \log(H / \delta) R^2\log T}{\lambda^4 T_{\max}} + \frac{R^2 \sqrt{\log(H / \delta)} \log T \log T_{\max}}{\lambda m^{-1/2}} + \frac{R^2 T_{\max} \log (H / \delta) \log T}{H \lambda m} \nonumber \\
        &\quad + \frac{c_\gamma^2 \log(H / \delta) \tau_{\mathrm{mix}} \log T_{\max} \log T}{\lambda^2 H}\Bigg)
    \end{align}
\end{proof}

%% file: Sections/Appendix/npg_analysis.tex
\section{NPG Analysis}
\label{appendix: npg analysis}
We now provide the analysis for the convergence of the Natural Policy Gradient (NPG) subroutine in Algorithm~\ref{alg:pdnac_nc}. 
\begin{lemma}
    \label{lemma: npg single sample estimate errors}
    Consider the single sample estimates for the Fisher information matrix and the policy gradient \(\hat{F}(\theta_k; z)\) and \(\hat{\nabla}_\theta J_g(\theta_k, \xi_{g}^k; z)\). If Assumptions~\ref{assump:score} and~\ref{assump:critic_approx} hold, then the following bounds on the squared bias and variance hold:
    \begin{align}
        \left \|\mathbb{E}_{\theta_k}\left[\hat{F}(\theta_k; z)\right] - F(\theta_k)\right\|^2 &\leq \delta_F^2 = 0 \\
        \left \|\mathbb{E}_{k}\left[\hat{F}(\theta_k; z)\right] - F(\theta_k)\right\|^2 &\leq \bar{\delta}_F^2 = 0 \\
        \mathbb{E}_{\theta_k} \left[\left\|\hat{F}(\theta_k; z) - F(\theta_k)\right\|^2\right] &\leq \sigma_F^2 = 2G_1^4 \\
        \mathbb{E}_k\left[\left\|\hat{F}(\theta_k; z) - F(\theta_k)\right\|^2\right] &\leq \bar{\sigma}_F^2 = 2G_1^4 \\
        \left \|\mathbb{E}_{\theta_k}\left[\hat{\nabla}_\theta J_g(\theta_k; z)\right] - \nabla_\theta J_g(\theta_k)\right\|^2 &\leq \delta_J^2 = \mathcal{O}\left(C_1^2 G_1^2 \|\xi_g^k - \xi_{g,*}^k\|^2 + G_1^2 C_4^2 \log(H / \delta) m^{-1} + G_1^2 \tau_{\mathrm{mix}}^2 \epsilon_{\mathrm{app}}\right)\\
        \left \|\mathbb{E}_{k}\left[\hat{\nabla}_\theta J_g(\theta_k; z)\right] - \nabla_\theta J_g(\theta_k)\right\|^2 &\leq \bar{\delta}_J^2 = \mathcal{O}\left(C_1^2 G_1^2 \|\mathbb{E}_k\left[\xi_g^k\right] - \xi_{g,*}^k\|^2 + G_1^2 C_4^2 \log(H / \delta) m^{-1} + G_1^2 \tau_{\mathrm{mix}}^2 \epsilon_{\mathrm{app}}\right)\\
        \mathbb{E}_{\theta_k} \left \|\left[\hat{\nabla}_\theta J_g(\theta_k; z)\right] - \nabla_\theta J_g(\theta_k)\right\|^2 &\leq \sigma_J^2 = \mathcal{O}\left(G_1^2 C_1'^2 \log(H / \delta)\right) \\
        \mathbb{E}_{k} \left \|\left[\hat{\nabla}_\theta J_g(\theta_k; z)\right] - \nabla_\theta J_g(\theta_k)\right\|^2 &\leq \bar{\sigma}_J^2 = \mathcal{O}\left(G_1^2 C_1'^2 \log(H / \delta)\right)
    \end{align}
\end{lemma}
\begin{proof}
    By the definition of the single sample estimate for the Fisher Information Matrix, we have
    \begin{align}
        \left\|\mathbb{E}_{\theta_k}\left[\hat{F}(\theta_k; z)\right] - F(\theta_k)\right\|^2 = 0
    \end{align}
    as the estimate is unbiased. We also have 
    \begin{align}
        \left\|\hat{F}(\theta_k; z) - F(\theta_k)\right\|^2 &\leq 2 \|\hat{F}(\theta_k; z)\|^2 \\
        &\leq 2 \|\nabla_{\theta_k} \log \pi_{\theta}(a | s)\|^4 \\
        &\leq 2 G_1^4
    \end{align}
    For the policy gradient estimate $\hat{\nabla}_\theta J_g(\theta_k, \xi^k_g; z)$, we have
    \begin{align}
        \|\hat{\nabla}_\theta J_g(\theta_k, \xi^k_g; z)\|^2 &= \hat{A}^{\pi_{\theta_k}}_g(\xi^k_g; z)^2 \|\nabla_\theta \log \pi_{\theta_k}(a|s)\|^2\nonumber\\
        &\leq G_1^2 \left(|g(s, a)| + |\eta_g^k| + \left|Q_g(\phi(s', a'); \zeta_g^k) - Q_g(\phi(s, a); \zeta_g^k)\right|\right)^2 \\
        &\leq \mathcal{O}\left(G_1^2 C_1'^2 \log(H / \delta)\right)
    \end{align}
    where the last inequality follows from the boundedness of the neural network output in the NTK regime (Lemma~\ref{lemma: ntk critic bounds}). Thus, the single-sample variance is bounded by a constant independent of $T_{\max}$. This results in the following bound
    \begin{align}
        \mathbb{E}_{\theta_k} \left[\left\|\hat{\nabla}_{\theta} J_g(\theta_k, \xi_{g}^k; z) - \nabla_{\theta} J_g(\theta_k)\right\|^2\right] \leq \mathcal{O}\left(G_1^2 C_1'^2 \log(H / \delta)\right)
    \end{align}

    We now analyze the bias of the estimator. By definition, we have $\xi_{g,*}^k = (\eta_{g,*}^k, \zeta_{g,*}^k)$ as the optimal parameters in the linearized function class $\mathcal{F}_{R,m}$. We decompose the expected error $\mathbb{E}_{\theta_k} [\hat{\nabla}_\theta J_g(\theta_k, \xi_g^k; z)] - \nabla_\theta J_g(\theta_k)$ into four terms $T_0, T_{0,\text{lin}}, T_1, T_2$:
    \begin{align}
        &\mathbb{E}_{\theta_k} \left[ \hat{\nabla}_\theta J_g(\theta_k, \xi_g^k; z) \right] - \nabla_\theta J_g(\theta_k) \nonumber \\
        &= \mathbb{E}_{\theta_k} \Bigg[ \Bigg\{ \underbrace{ (\eta_{g,*}^k - \eta_g^k) + \left( \hat{Q}_g(\phi(s', a') ; \zeta_g^k) - \hat{Q}_g(\phi(s', a') ; \zeta_{g,*}^k) \right) - \left( \hat{Q}_g(\phi(s, a); \zeta_g^k) - \hat{Q}_g(\phi(s, a); \zeta_{g,*}^k) \right) }_{T_0: \text{Optimization Error (Linearized)}} \nonumber \\
        &\quad + \underbrace{ \left( Q_g(\phi(s', a'); \zeta_g^k) - \hat{Q}_g(\phi(s', a') ; \zeta_g^k) \right) - \left( Q_g(\phi(s, a); \zeta_g^k) - \hat{Q}_g(\phi(s, a); \zeta_g^k) \right) }_{T_{0,\text{lin}}: \text{Linearization Residuals}} \nonumber \\
        &\quad + \underbrace{ Q_g^{\pi_{\theta_k}}(s', a') - \hat{Q}_g(\phi(s', a'); \zeta_{g,*}^k) - \left( Q_g^{\pi_{\theta_k}}(s, a) - \hat{Q}_g(\phi(s, a); \zeta_{g,*}^k) \right)}_{T_1: \text{Function Approximation Error}} \nonumber \\
        &\quad + \underbrace{ g(s,a) - \eta_{g, *}^k + Q_g^{\pi_{\theta_k}}(s', a') - Q_g^{\pi_{\theta_k}}(s, a) }_{T_2: \text{True Advantage / Bellman Error}} \Bigg\} \nabla_\theta \log \pi_{\theta_k}(a|s) \Bigg] - \nabla_\theta J_g(\theta_k)
    \end{align}
    Using the definition of the advantage function $A_g^{\pi}(s,a)$ and the Policy Gradient Theorem:
    \begin{align}
        &\mathbb{E}_{\theta_k} \left[ T_2 \nabla_\theta \log \pi_{\theta_k}(a|s) \right] - \nabla_\theta J_g(\theta_k) \nonumber 
        \\
        &= \mathbb{E}_{s,a \sim d^{\pi_{\theta_k}}} \left[ \mathbb{E}_{s'}\left[ g(s,a) - \eta_{g, *}^k + Q_g^{\pi_{\theta_k}}(s') - Q_g^{\pi_{\theta_k}}(s) \;\middle|\; s, a \right] \nabla_\theta \log \pi_{\theta_k}(a|s) \right] - \nabla_\theta J_g(\theta_k) \\
        &\stackrel{(a)}{=} \mathbb{E}_{s,a \sim d^{\pi_{\theta_k}}} \left[ A_g^{\pi_{\theta_k}}(s, a) \nabla_\theta \log \pi_{\theta_k}(a|s) \right] - \nabla_\theta J_g(\theta_k) \\
        &= 0
    \end{align}
    where \((a)\) uses the Bellman equation and that \(J_g(\theta_k) = \eta_{g, *}^k\) by definition. Using the linearity of $\hat{Q}_g$ with respect to $\zeta$ (defined via inner product with $\nabla_\zeta Q_g(\cdot; \zeta_{g, 0})$):
    \begin{align}
        T_0 &= (\eta_{g,*}^k - \eta_g^k) + \left( \hat{Q}_g(\phi(s', a'); \zeta_g^k) - \hat{Q}_g(\phi(s', a'); \zeta_{g,*}^k) \right) - \left( \hat{Q}_g(\phi(s, a); \zeta_g^k) - \hat{Q}_g(\phi(s, a); \zeta_{g,*}^k) \right) \nonumber \\
        &= (\eta_{g,*}^k - \eta_g^k) + \left\langle \nabla_\zeta Q_g(\phi(s', a'); \zeta_0) - \nabla_\zeta Q_g(\phi(s', a'); \zeta_0), \; \zeta_g^k - \zeta_{g,*}^k \right\rangle \\
        &\leq 2C_1 \|\xi_{g}^k - \xi_{g,*}^k\|
    \end{align}
    where the last inequality uses Lemma~\ref{lemma: ntk critic bounds}. Applying the Cauchy-Schwarz inequality and taking the squared expectation:
    \begin{align}
        \left\| \mathbb{E}_{\theta_k} [ T_0 \nabla_\theta \log \pi ] \right\|^2 
        &\leq G_1^2 \mathbb{E}_{\theta_k} \left[ |T_0|^2 \right] \leq \mathcal{O}\left(C_1^2 G_1^2 \|\xi_g^k - \xi_{g,*}^k\|^2 \right)
    \end{align}
    Using Lemma~\ref{lemma: ntk critic bounds}, which bounds the difference between the neural network and its linearization in the NTK regime by $\mathcal{O}(m^{-1/2})$:
    \begin{align}
        |T_{0,\text{lin}}| 
        &= \left| \left( Q_g(\phi(s', a'); \zeta_g^k) - \hat{Q}_g(\phi(s', a'); \zeta_g^k) \right) - \left( Q_g(\phi(s, a); \zeta_g^k) - \hat{Q}_g(\phi(s, a); \zeta_g^k) \right) \right| \nonumber \\
        &\leq \left| Q_g(\phi(s', a'); \zeta_g^k) - \hat{Q}_g(\phi(s', a'); \zeta_g^k) \right| + \left| Q_g(\phi(s, a); \zeta_g^k) - \hat{Q}_g(\phi(s, a); \zeta_g^k) \right| \nonumber \\
        &\leq \mathcal{O}\left(C_4 m^{-1/2} \sqrt{\log (H / \delta)}\right)
    \end{align}
    Thus, the squared error contribution is:
    \begin{align}
        \left\| \mathbb{E}_{\theta_k} [ T_{0,\text{lin}} \nabla_\theta \log \pi ] \right\|^2 
        \leq G_1^2 \mathbb{E}_{\theta_k} [ |T_{0,\text{lin}}|^2 ] 
        \leq \mathcal{O}\left( \frac{G_1^2 C_4^2\log(H / \delta)}{m} \right)
    \end{align}
    Using Assumption 3.4 (Neural Critic Approximation Error), which bounds the expected squared Bellman error of the optimal linearized critic $\xi_{g,*}^k$:
    \begin{align}
        \mathbb{E}_{\theta_k} [ T_1^2 ] 
        &= \mathbb{E}_{\theta_k} \left[ \left( \left( Q_g^{\pi}(s', a') - \hat{Q}_g(\phi(s', a'); \zeta_{g,*}^k) \right) - \left( Q_g^{\pi}(s,a) - \hat{Q}_g(\phi(s, a); \zeta_{g,*}^k) \right)\right)^2 \right] \nonumber \\
        &\leq \mathcal{O}\left(\tau_{\mathrm{mix}}^2 \epsilon_{\mathrm{app}}\right)
    \end{align}
    using the average reward analogue of A.3 in \citet{ke2024improved}. We have
    \begin{align}
        \left\| \mathbb{E}_{\theta_k} [ T_1 \nabla_\theta \log \pi ] \right\|^2 
        \leq G_1^2 \mathbb{E}_{\theta_k} [ T_1^2 ] 
        \leq \mathcal{O}\left( G_1^2 \tau_{\mathrm{mix}}^2\epsilon_{\text{app}} \right)
    \end{align}
    Combining these terms, we have
    \begin{align}
        \left\|\mathbb{E}_{\theta_k} \left[ \hat{\nabla}_\theta J_g(\theta_k, \xi_g^k; z) \right] - \nabla_\theta J_g(\theta_k)\right\|^2 &\leq \mathcal{O}\left(C_1^2 G_1^2 \|\xi_g^k - \xi_{g,*}^k\|^2 + G_1^2 C_4^2 \log(H / \delta) m^{-1} + G_1^2 \tau_{\mathrm{mix}}^2 \epsilon_{\mathrm{app}}\right)
    \end{align}
    Applying the same argument yields
    \begin{align}
        \left\|\mathbb{E}_k\mathbb{E}_{\theta_k} \left[ \hat{\nabla}_\theta J_g(\theta_k, \xi_g^k; z) \right] - \nabla_\theta J_g(\theta_k)\right\|^2 &\leq \mathcal{O}\left(C_1^2 G_1^2 \|\mathbb{E}\left[\xi_g^k\right] - \xi_{g,*}^k\|^2 + G_1^2 C_4^2 \log(H / \delta) m^{-1} + G_1^2 \tau_{\mathrm{mix}}^2 \epsilon_{\mathrm{app}}\right)
    \end{align}
\end{proof}

\begin{lemma}
    \label{lemma: npg mlmc estimate errors}
    We denote the MLMC estimators for the Fisher Information Matrix and the policy gradient as 
    \begin{align}
        F_{kh}^{\mathrm{MLMC}} &= F_{kh}^0 + 2^{Q_h^k} \left(F_{kh}^{Q_h^k} - F_{kh}^{Q_h^k-1}\right)\mathbb{I}\left(2^{Q_h^k} \leq T_{\max}\right) \\
        \nabla_{\theta} J_{g, kh}^{\mathrm{MLMC}} &= \nabla_\theta J_{g, kh}^0 + 2^{Q_h^k} \left(\nabla_\theta J_{g, kh}^{Q_h^k} - \nabla_\theta J_{g, kh}^{Q_h^k-1}\right)\mathbb{I}\left(2^{Q_h^k} \leq T_{\max}\right)
    \end{align}
    where \(Q_h^k \sim \mathrm{Geom}(1/2)\), \(F_{kh}^j = \frac{1}{2^j} \sum_{i=0}^{2^j-1} \hat{F}(\theta_k; z_{kh}^i)\), and \(\nabla_{\theta} J_{g, kh}^j = \frac{1}{2^j} \sum_{i=0}^{2^j-1} \hat{\nabla}_\theta J_g(\theta_k; \xi_{g}^k; z_{kh}^i)\). Then the following bounds hold:
    \begin{align}
        \left\|\mathbb{E}_{k, h} \left[F_{kh}^{\mathrm{MLMC}}\right] - F(\theta_k)\right\|^2 &\leq \mathcal{O}\left(G_1^4 \tau_{\mathrm{mix}} T_{\max}^{-1}\right) \\
        \mathbb{E}_{k, h} \left[\left\|F_{kh}^{\mathrm{MLMC}} - F(\theta_k) \right\|^2\right] &\leq \mathcal{O}\left(G_1^4 \tau_{\mathrm{mix}} \log T_{\max}\right) \\
        \left\|\mathbb{E}_{k, h} \left[\nabla_\theta J_{g, kh}^{\mathrm{MLMC}}\right] - \nabla_\theta J_g(\theta_k, \xi_g^k)\right\|^2 &\leq \mathcal{O}\Bigg(G_1^2 \log(H / \delta) \tau_{\mathrm{mix}} T_{\max}^{-1} + G_1^2 \|\xi_g^k - \xi_{g,*}^k\|^2 \nonumber \\
        &\quad + G_1^2 \log(H / \delta) m^{-1} + G_1^2 \tau_{\mathrm{mix}}^2 \epsilon_{\mathrm{app}}\Bigg) \\
        \mathbb{E}_{k, h} \left[\left\|\nabla_\theta J_{g, kh}^{\mathrm{MLMC}} - \nabla_\theta J_g(\theta_k, \xi_g^k) \right\|^2\right] &\leq \mathcal{O}\Bigg(G_1^2 \log(H / \delta) \tau_{\mathrm{mix}} \log T_{\max} + G_1^2 \|\xi_g^k - \xi_{g,*}^k\|^2 \nonumber \\
        &\quad + G_1^2 \log(H / \delta) m^{-1} + G_1^2 \tau_{\mathrm{mix}}^2 \epsilon_{\mathrm{app}}\Bigg) \\
        \left\|\mathbb{E}_{k} \left[\nabla_\theta J_{g, kh}^{\mathrm{MLMC}}\right] - \nabla_\theta J_g(\theta_k, \xi_g^k)\right\|^2 &\leq \mathcal{O}\Bigg(G_1^2 \log(H / \delta) \tau_{\mathrm{mix}} T_{\max}^{-1} + G_1^2 \|\mathbb{E}_k\left[\xi_g^k\right] - \xi_{g,*}^k\|^2 \nonumber \\
        &\quad + G_1^2 \log(H / \delta) m^{-1} + G_1^2 \tau_{\mathrm{mix}}^2 \epsilon_{\mathrm{app}}\Bigg)
    \end{align}
\end{lemma}

\begin{proof}
    Using Lemma~\ref{lemma: npg single sample estimate errors} with Lemma~\ref{lemma: mlmc_properties}, we have
    \begin{align}
        \left\|\mathbb{E}_{k, h} \left[F_{kh}^{\mathrm{MLMC}}\right] - F(\theta_k)\right\|^2 &\leq \mathcal{O}\left(\sigma_F^2 \tau_{\mathrm{mix}} T_{\max}^{-1} + \delta_F^2\right) = \mathcal{O}\left(G_1^4 \tau_{\mathrm{mix}} T_{\max}^{-1}\right) \\
        \mathbb{E}_{k, h} \left[\left\|F_{kh}^{\mathrm{MLMC}} - F(\theta_k) \right\|^2\right] &\leq\mathcal{O}\left(\sigma_F^2 \tau_{\mathrm{mix}} \log T_{\max} + \delta_F^2\right) = \mathcal{O}\left(G_1^4 \tau_{\mathrm{mix}} \log T_{\max}\right)
    \end{align}
    Similarly, we have
    \begin{align}
        \left\|\mathbb{E}_{k, h} \left[\nabla_\theta J_{g, kh}^{\mathrm{MLMC}}\right] - \nabla_\theta J_g(\theta_k, \xi_g^k)\right\|^2 &\leq \mathcal{O}\left(\sigma_J^2 \tau_{\mathrm{mix}} T_{\max}^{-1} +\delta_J^2 \right) \\ &= \mathcal{O}\Bigg(G_1^2 \log(H / \delta) \tau_{\mathrm{mix}} T_{\max}^{-1} + G_1^2 \|\xi_g^k - \xi_{g,*}^k\|^2 \nonumber \\
        &\quad + G_1^2 \log(H / \delta) m^{-1} + G_1^2 \tau_{\mathrm{mix}}^2 \epsilon_{\mathrm{app}}\Bigg) \\
        \mathbb{E}_{k, h} \left[\left\|\nabla_\theta J_{g, kh}^{\mathrm{MLMC}} - \nabla_\theta J_g(\theta_k, \xi_g^k) \right\|^2\right] &\leq \mathcal{O}\left(\sigma_J^2 \tau_{\mathrm{mix}} \log T_{\max} +\delta_J^2 \right) \\ &=  \mathcal{O}\Bigg(G_1^2 \log(H / \delta) \tau_{\mathrm{mix}} \log T_{\max} + G_1^2 \|\xi_g^k - \xi_{g,*}^k\|^2 \nonumber \\
        &\quad + G_1^2 \log(H / \delta) m^{-1} + G_1^2 \tau_{\mathrm{mix}}^2 \epsilon_{\mathrm{app}}\Bigg) \\
        \left\|\mathbb{E}_{k} \left[\nabla_\theta J_{g, kh}^{\mathrm{MLMC}}\right] - \nabla_\theta J_g(\theta_k, \xi_g^k)\right\|^2 &\leq \mathcal{O}\left(\bar{\sigma}_J^2 \tau_{\mathrm{mix}} T_{\max}^{-1} + \bar{\delta}_J^2 \right) \\ &=  \mathcal{O}\Bigg(G_1^2 \log(H / \delta) \tau_{\mathrm{mix}} T_{\max}^{-1} + G_1^2 \|\mathbb{E}_k\left[\xi_g^k\right] - \xi_{g,*}^k\|^2 \nonumber \\
        &\quad + G_1^2 \log(H / \delta) m^{-1} + G_1^2 \tau_{\mathrm{mix}}^2 \epsilon_{\mathrm{app}}\Bigg)
    \end{align}
\end{proof}

\begin{theorem}[NPG Estimation Error Bounds]
\label{thm:npg_estimation_error_bounds}
Suppose the assumptions in Section~\ref{sec: assumptions} hold. Let the NPG step size be set as $\gamma_\omega = \frac{2 \log T}{\mu H}$ and the initialization be $\omega_{g,0}^k = 0$. Let $\omega_{g,*}^k = F(\theta_k)^{-1} \nabla_\theta J_g(\theta_k)$ denote the true Natural Policy Gradient direction. Then, after $H$ inner iterations, the mean-squared error of the estimator $\omega_g^k$ satisfies:
\begin{align}
    \mathbb{E}_k \left[\left\|\omega_g^k - \omega_{g, *}^k\right\|^2\right] &\leq \frac{G_1^2 \tau_{\mathrm{mix}}^2}{\mu^2 T^2} + \tilde{\mathcal{O}}\left( \frac{G_1^6 \tau_{\mathrm{mix}}^3}{\mu^4} \left(\frac{1}{H} + \frac{1}{T_{\max}}\right) + \frac{G_1^2}{\mu^2 \sqrt{m}} + \frac{G_1^2 \tau_{\mathrm{mix}}^2}{\mu^2} \epsilon_{\mathrm{app}} \right),
\end{align}
and the squared bias of the estimator satisfies:
\begin{align}
    \left\| \mathbb{E}_k [\omega_g^k] - \omega_{g,*}^k \right\|^2 &\leq \frac{G_1^2 \tau_{\mathrm{mix}}^2}{\mu^2 T^2} + \tilde{\mathcal{O}}\left( \frac{G_1^6 \tau_{\mathrm{mix}}^3}{\mu^4 T_{\max}} + \frac{G_1^2}{\mu^2 \sqrt{m}} + \frac{G_1^2 \tau_{\mathrm{mix}}^2}{\mu^2} \epsilon_{\mathrm{app}} \right),
\end{align}
where $\tilde{\mathcal{O}}$ hides polylogarithmic factors in $H, T, T_{\max}$, and $1/\delta$.
\end{theorem}

\begin{proof}
    We note that by Lemma~\ref{lemma: npg mlmc estimate errors}, that the conditions of Equation 36 in Theorem B.1 in \citet{xu2025global} are satisfied. Additionally, for any policy parameter \(\theta\), we have
    \begin{align}
        \mu \leq \|F(\theta)\| \leq \mathcal{O}\left(G_1\right), \quad \|\nabla_\theta J(\theta)\| \leq \mathcal{O}\left(G_1 \tau_{\mathrm{mix}}\right)
    \end{align}
    from Theorem 4.7 in \citet{xu2025global}. Thus, all conditions of Theorem B.1 in \citet{xu2025global} are satisfied, and we have
    \begin{align*}
        &\mathbb{E}_k \left[\left\|\omega_g^k - \omega_{g, *}^k\right\|^2\right] \leq \exp\left(-H \gamma_\omega \mu\right) \left\|\omega_0 - \omega_{g, *}^k\right\|^2 + \mathcal{O}\!\left(\gamma_\omega R_0 + R_1\right) \\
        &\left\| \mathbb{E}_k [\omega_g^k] - \omega_{g,k}^* \right\|^2 \nonumber\\
        &\leq \exp (-H \gamma_\omega \mu) \left\| \omega_0 - \omega_{g,k}^* \right\|^2 + \frac{\mu \gamma_\omega + 1}{\mu^2} \left[ \mathcal{O}\!\left(\frac{G_1^4 \tau_{\text{mix}}}{T_{\max}}\right) \!\left\{ \left\| \omega_0 - \omega_{g,k}^* \right\|^2 + \mathcal{O} (\gamma_\omega R_0 + R_1) \right\} \!+\! \bar{R}_1 \right]
    \end{align*}
    where
    \begin{align}
        R_0 &= \tilde{\mathcal{O}}\left(\mu^{-3} G_1^6 \tau_{\mathrm{mix}}^3 + \mu^{-1} G_1^2 \log(H / \delta) (\tau_{\mathrm{mix}} + m^{-1}) +  \mu^{-1} G_1^2 \mathbb{E}_k\|\xi_g^k - \xi_{g, *}^k\|^2 + \mu^{-1} G_1^2 \tau_{\mathrm{mix}}^2 \epsilon_{\mathrm{app}}\right) \\
        R_1 &= \mathcal{O}\left(\mu^{-4}G_1^6 \tau_{\mathrm{mix}}^3 T_{\max}^{-1} + \mu^{-2}G_1^2\log(H / \delta) (\tau_{\mathrm{mix}} T_{\max}^{-1} + m^{-1}) + \mu^{-2} G_1^2 \mathbb{E}_k \|\xi_g^k - \xi_{g, *}^k\|^2 + \mu^{-2}G_1^2 \tau_{\mathrm{mix}}^2 \epsilon_{\mathrm{app}}\right) \\
        \bar{R}_1 &= \mathcal{O}\left(\mu^{-2}G_1^6 \tau_{\mathrm{mix}}^3 T_{\max}^{-1} + G_1^2 \log(H / \delta) (\tau_{\mathrm{mix}} T_{\max}^{-1} + m^{-1}) + G_1^2 \|\mathbb{E}_k\left[\xi_g^k\right] - \xi_{g, *}^k\|^2 + G_1^2 \tau_{\mathrm{mix}}^2 \epsilon_{\mathrm{app}}\right)
    \end{align}
    and we set \(\gamma_\omega = \frac{2 \log T}{\mu H}\). Using \(\omega_0 = 0\) and \(\|\omega_{g, *}^k\|^2 = \|F(\theta_k)^{-1} \nabla_\theta J(\theta_k)\|^2 \leq \mathcal{O}\left(\mu^{-1} G_1 \tau_{\mathrm{mix}}\right)\), this yields
    \begin{align*}
        \mathbb{E}_k \left[\left\|\omega_g^k - \omega_{g, *}^k\right\|^2\right] &\leq \frac{G_1^2 \tau_{\mathrm{mix}}^2}{\mu^2 T^2} + \tilde{\mathcal{O}}\Bigg( \left(H^{-1} + T_{\max}^{-1}\right) \left(\mu^{-4} G_1^6 \tau_{\mathrm{mix}}^3 + \mu^{-2} G_1^2 \log(H / \delta) \tau_{\mathrm{mix}} \right)\\
        &\quad\quad + \mu^{-2}G_1^2 m^{-1/2} + \mu^{-2} G_1^2 \mathbb{E}_k \|\xi_g^k - \xi_{g, *}^k\|^2 + \mu^{-2}G_1^2 \tau_{\mathrm{mix}}^2 \epsilon_{\mathrm{app}} \Bigg) \\
        \left\| \mathbb{E}_k [\omega_g^k] - \omega_{g,k}^* \right\|^2 &\leq \left(\frac{G_1^2 \tau_{\mathrm{mix}}^2}{\mu^2 T^2} + \frac{G_1^6 \tau_{\text{mix}}^3}{\mu^4 T_{\max}}\right) +  \tilde{\mathcal{O}}\bigg(\mu^{-4}G_1^6 \tau_{\mathrm{mix}}^3 T_{\max}^{-1} + \mu^{-2} G_1^2 \log(H / \delta) \tau_{\mathrm{mix}} T_{\max}^{-1} + \mu^{-2}G_1 m^{-1/2}\\
        &\quad\quad + \mu^{-2}G_1^2 \|\mathbb{E}_k\left[\xi_g^k\right] - \xi_{g, *}^k\|^2 + \mu^{-2}G_1^2 \tau_{\mathrm{mix}}^2 \epsilon_{\mathrm{app}} + \frac{G_1^6 \tau_{\mathrm{mix}}^2}{\mu^4 T_{\max}} \mathbb{E}_k \|\xi_g^k - \xi_{g, *}^k\|^2\bigg) \\
        &\leq \left(\frac{G_1^2 \tau_{\mathrm{mix}}^2}{\mu^2 T^2} + \frac{G_1^6 \tau_{\text{mix}}^3}{\mu^4 T_{\max}}\right) +  \tilde{\mathcal{O}}\bigg(\mu^{-4}G_1^6 \tau_{\mathrm{mix}}^3 T_{\max}^{-1} + \mu^{-2} G_1^2 \log(H / \delta) \tau_{\mathrm{mix}} T_{\max}^{-1} + \mu^{-2}G_1 m^{-1/2}\\
        &\quad\quad + \mu^{-2}G_1^2 \mathbb{E}_k\left[\|\xi_g^k - \xi_{g, *}^k\|^2\right] + \mu^{-2}G_1^2 \tau_{\mathrm{mix}}^2 \epsilon_{\mathrm{app}} + \frac{G_1^6 \tau_{\mathrm{mix}}^2}{\mu^4 T_{\max}} \mathbb{E}_k \|\xi_g^k - \xi_{g, *}^k\|^2\bigg) \\
        &\leq \left(\frac{G_1^2 \tau_{\mathrm{mix}}^2}{\mu^2 T^2} + \frac{G_1^6 \tau_{\text{mix}}^3}{\mu^4 T_{\max}}\right) +  \tilde{\mathcal{O}}\bigg(\mu^{-4}G_1^6 \tau_{\mathrm{mix}}^3 T_{\max}^{-1} + \mu^{-2} G_1^2 \log(H / \delta) \tau_{\mathrm{mix}} T_{\max}^{-1} + \mu^{-2}G_1 m^{-1/2}\\
        &\quad\quad + \mu^{-2}G_1^2 \mathbb{E}_k\left[\|\xi_g^k - \xi_{g, *}^k\|^2\right] + \mu^{-2}G_1^2 \tau_{\mathrm{mix}}^2 \epsilon_{\mathrm{app}}\bigg) \\
    \end{align*}
    Finally, we substitute the value for \(\mathbb{E}_k\left[\|\xi_g^k - \xi_{g, *}^k\|^2\right]\) to yield
    \begin{align}
        \mathbb{E}_k \left[\left\|\omega_g^k - \omega_{g, *}^k\right\|^2\right] &\leq \frac{G_1^2 \tau_{\mathrm{mix}}^2}{\mu^2 T^2} + \tilde{\mathcal{O}}\Bigg( \left(\frac{1}{H} + \frac{1}{T_{\max}}\right) \left(\frac{G_1^6 \tau_{\mathrm{mix}}^3}{\mu^4} + \frac{G_1^2 \log(H / \delta) \tau_{\mathrm{mix}}}{\mu^2} \right) \nonumber \\
        &\quad\quad + \frac{G_1^2}{\mu^2 \sqrt{m}} + \frac{G_1^2 c_\gamma^2 \log(H / \delta) \tau_{\mathrm{mix}}}{\lambda^2 \mu^2 H} + \frac{G_1^2 \tau_{\mathrm{mix}}^2}{\mu^2} \epsilon_{\mathrm{app}} \Bigg) \\
        \left\| \mathbb{E}_k [\omega_g^k] - \omega_{g,k}^* \right\|^2 &\leq \frac{G_1^2 \tau_{\mathrm{mix}}^2}{\mu^2 T^2} +  \tilde{\mathcal{O}}\bigg(\frac{G_1^6 \tau_{\mathrm{mix}}^3}{ T_{\max} \mu^4} + \frac{G_1^2 \log(H / \delta) \tau_{\mathrm{mix}}}{\mu^2 T_{\max}} + \frac{G_1^2}{\mu^2 \sqrt{m}} \nonumber \\
        &\quad\quad + \frac{G_1^2 c_\gamma^2 \log(H / \delta)}{\mu^2 \lambda^2 H} + \frac{G_1^2 \tau_{\mathrm{mix}}^2}{\mu^2} \epsilon_{\mathrm{app}}\bigg)
    \end{align}
\end{proof}

%% file: Sections/Appendix/cmdp_analysis.tex
\section{CMDP Analysis}
In this section, we formally analyze the convergence rate of the overall CMDP.

\begin{theorem}
\label{thm: cmdp convergence}
Suppose the assumptions in Section~\ref{sec: assumptions} hold. Let the step sizes be $\alpha = \beta = \Theta(T^{-1/4})$, the number of outer iterations be $K = \Theta(T^{1/2})$, and the number of inner iterations be $H = \Theta(T^{1/2})$. Let the MLMC truncation level be $T_{\max} = T$. Then, the sequence of policies $\{\theta_k\}_{k=0}^{K-1}$ generated by Algorithm 1 satisfies the following bounds on the average reward gap and constraint violation:

\begin{align}
    \frac{1}{K} \mathbb{E} \left[\sum_{k=0}^{K-1} \left( J_r^{\pi^*} - J_r(\theta_k) \right)\right] &\leq \tilde{\mathcal{O}}\left(\sqrt{\epsilon_{\mathrm{bias}}} + \sqrt{\epsilon_{\mathrm{app}}} + T^{-1/4} + m^{-1/4}\right), \\
    \frac{1}{K} \mathbb{E} \left[\sum_{k=0}^{K-1} - J_c(\theta_k) \right] &\leq \tilde{\mathcal{O}}\left(\sqrt{\epsilon_{\mathrm{bias}}} + \sqrt{\epsilon_{\mathrm{app}}} + T^{-1/4} + m^{-1/4}\right),
\end{align}
where $\tilde{\mathcal{O}}$ hides polylogarithmic factors in $T$.
\end{theorem}

\begin{proof}
    Combining Theorem~\ref{thm:npg_estimation_error_bounds} with the fact that \(\lambda_k \in [0, 2 / \delta_s]\), we have
    \begin{align}
        \mathbb{E}\left[\|\mathbb{E}_k\left[\omega_k\right] - \omega_k^*\|\right] &\leq \left(1 + \frac{2}{\delta_s}\right) \tilde{\mathcal{O}}\Bigg(\frac{G_1 \tau_{\mathrm{mix}}}{\mu T} + \frac{G_1^3 \tau_{\mathrm{mix}}^2}{\sqrt{T_{\max}}\mu^2} + \frac{G_1 \sqrt{\log(H / \delta)} \tau_{\mathrm{mix}}}{\mu \sqrt{T_{\max}}} + \frac{G_1}{\mu m^{1/4}} \nonumber \\
        &\quad + \frac{G_1 c_\gamma \sqrt{\log(H / \delta)}}{\mu \lambda \sqrt{H}} + \frac{G_1 \tau_{\mathrm{mix}}}{\mu}\sqrt{\epsilon_{\mathrm{app}}}\Bigg) \\
        \mathbb{E}_k \left[\left\|\omega^k - \omega_{ k}^*\right\|^2\right] &\leq \left(1 + \frac{4}{\delta_s^2}\right)\tilde{\mathcal{O}}\Bigg( \frac{G_1^2 \tau_{\mathrm{mix}}^2}{\mu^2 T^2} + \left(\frac{1}{H} + \frac{1}{T_{\max}}\right) \left(\frac{G_1^6 \tau_{\mathrm{mix}}^3}{\mu^4} + \frac{G_1^2 \log(H / \delta) \tau_{\mathrm{mix}}}{\mu^2} \right) \nonumber \\
        &\quad\quad + \frac{G_1^2}{\mu^2 \sqrt{m}} + \frac{G_1^2 c_\gamma^2 \log(H / \delta) \tau_{\mathrm{mix}}}{\lambda^2 \mu^2 H} + \frac{G_1^2 \tau_{\mathrm{mix}}^2}{\mu^2} \epsilon_{\mathrm{app}} \Bigg)   
    \end{align}
    Using Lemma 4.6 and Lemma G.2 in \citet{xu2025global} and setting \(T_{\max} = T\), we have
    \begin{align}
    &\frac{1}{K} \mathbb{E} \left[\sum_{k=0}^{K-1} \left( \mathcal{L}(\pi^*, \lambda_k) - \mathcal{L}(\theta_k, \lambda_k) \right) \right] \nonumber \\
    &\leq \sqrt{\epsilon_{\mathrm{bias}}} + \frac{1}{\alpha K} \mathbb{E}_{s \sim d^{\pi^*}} \left[ KL(\pi^*(\cdot|s) \| \pi_{\theta_0}(\cdot|s)) \right] \nonumber \\
    &\quad + \left(1 + \frac{2}{\delta_s}\right) \tilde{\mathcal{O}}\Bigg(\frac{G_1^3 \tau_{\mathrm{mix}}^2}{\sqrt{T}\mu^2} + \frac{G_1 \sqrt{\log(H / \delta)} \tau_{\mathrm{mix}}}{\mu \sqrt{T}} + \frac{G_1}{\mu m^{1/4}} + \frac{G_1 \tau_{\mathrm{mix}}}{\mu}\sqrt{\epsilon_{\mathrm{app}}} + \frac{G_1 c_\gamma \sqrt{\log(H / \delta)}}{\mu \lambda \sqrt{H}}\Bigg) \nonumber \\ 
    &\quad + \alpha G_2 \left(1 + \frac{4}{\delta_s^2}\right)\tilde{\mathcal{O}}\Bigg(\frac{G_1^6 \tau_{\mathrm{mix}}^3}{H\mu^4} + \frac{G_1^2 \log(H / \delta) \tau_{\mathrm{mix}}}{H\mu^2} + \frac{G_1^2}{\mu^2 \sqrt{m}} + \frac{G_1^2 c_\gamma^2 \log(H / \delta) \tau_{\mathrm{mix}}}{\lambda^2 \mu^2 H} + \frac{G_1^2 \tau_{\mathrm{mix}}^2}{\mu^2} \epsilon_{\mathrm{app}} \Bigg) \nonumber \\
    &\quad + \alpha \mathcal{O} \left( \left( 1 + \frac{4}{\delta_s^2} \right) \frac{G_1^2 \tau_{\mathrm{mix}}^2}{\mu^2} \right)
    \end{align}
    By the definition of the Lagrange function, we have
    \begin{align}
    &\frac{1}{K} \mathbb{E} \left[\sum_{k=0}^{K-1} \left( J_r^{\pi^*} - J_r(\theta_k) \right) \right] \nonumber \\
    &\leq \sqrt{\epsilon_{\mathrm{bias}}} + \frac{1}{\alpha K} \mathbb{E}_{s \sim d^{\pi^*}} \left[ KL(\pi^*(\cdot|s) \| \pi_{\theta_0}(\cdot|s)) \right] \nonumber \\
    &\quad + \left(1 + \frac{2}{\delta_s}\right) \tilde{\mathcal{O}}\Bigg(\frac{G_1^3 \tau_{\mathrm{mix}}^2}{\sqrt{T}\mu^2} + \frac{G_1 \sqrt{\log(H / \delta)} \tau_{\mathrm{mix}}}{\mu \sqrt{T}} + \frac{G_1}{\mu m^{1/4}} + \frac{G_1 \tau_{\mathrm{mix}}}{\mu}\sqrt{\epsilon_{\mathrm{app}}} + \frac{G_1 c_\gamma \sqrt{\log(H / \delta)}}{\mu \lambda \sqrt{H}}\Bigg) \nonumber \\ 
    &\quad + \alpha G_2 \left(1 + \frac{4}{\delta_s^2}\right)\tilde{\mathcal{O}}\Bigg(\frac{G_1^6 \tau_{\mathrm{mix}}^3}{H\mu^4} + \frac{G_1^2 \log(H / \delta) \tau_{\mathrm{mix}}}{H\mu^2} + \frac{G_1^2}{\mu^2 \sqrt{m}} + \frac{G_1^2 c_\gamma^2 \log(H / \delta) \tau_{\mathrm{mix}}}{\lambda^2 \mu^2 H} + \frac{G_1^2 \tau_{\mathrm{mix}}^2}{\mu^2} \epsilon_{\mathrm{app}} \Bigg) \nonumber \\
    &\quad + \alpha \mathcal{O} \left( \left( 1 + \frac{4}{\delta_s^2} \right) \frac{G_1^2 \tau_{\mathrm{mix}}^2}{\mu^2} \right) - \frac{1}{K} \mathbb{E} \left[\sum_{k=0}^{K-1} \lambda_k \left( J_c^{\pi^*} - J_c(\theta_k) \right) \right]
    \end{align}
    From Equation 75 in \citet{xu2025global}, we have
    \begin{align}
        - \frac{1}{K} \mathbb{E} \left[\sum_{k=0}^{K-1} \lambda_k \left( J_c^{\pi^*} - J_c(\theta_k) \right) \right] &\leq \frac{1}{K} \mathbb{E}\left[\sum_{k=0}^{K-1} |\lambda_k| \left(\eta_c^*(\theta_k) - \mathbb{E}\left[\eta_c^k\right]\right)\right] + \beta \\
        &\leq \tilde{\mathcal{O}}\left(\frac{G_1c_\gamma \sqrt{\log(H / \delta)}}{\mu \lambda \sqrt{H}} + \frac{G_1}{\mu m^{1/4}} + \beta\right)
    \end{align}
    Thus,
    \begin{align}
    &\frac{1}{K} \mathbb{E} \left[\sum_{k=0}^{K-1} \left( J_r^{\pi^*} - J_r(\theta_k) \right) \right] \nonumber \\
    &\leq \sqrt{\epsilon_{\mathrm{bias}}} + \frac{1}{\alpha K} \mathbb{E}_{s \sim d^{\pi^*}} \left[ KL(\pi^*(\cdot|s) \| \pi_{\theta_0}(\cdot|s)) \right] \nonumber \\
    &\quad + \left(1 + \frac{2}{\delta_s}\right) \tilde{\mathcal{O}}\Bigg(\frac{G_1^3 \tau_{\mathrm{mix}}^2}{\sqrt{T}\mu^2} + \frac{G_1 \sqrt{\log(H / \delta)} \tau_{\mathrm{mix}}}{\mu \sqrt{T}} + \frac{G_1}{\mu m^{1/4}} + \frac{G_1 \tau_{\mathrm{mix}}}{\mu}\sqrt{\epsilon_{\mathrm{app}}} + \frac{G_1 c_\gamma \sqrt{\log(H / \delta)}}{\mu \lambda \sqrt{H}}\Bigg) \nonumber \\ 
    &\quad + \alpha G_2 \left(1 + \frac{4}{\delta_s^2}\right)\tilde{\mathcal{O}}\Bigg(\frac{G_1^6 \tau_{\mathrm{mix}}^3}{H\mu^4} + \frac{G_1^2 \log(H / \delta) \tau_{\mathrm{mix}}}{H\mu^2} + \frac{G_1^2}{\mu^2 \sqrt{m}} + \frac{G_1^2 c_\gamma^2 \log(H / \delta) \tau_{\mathrm{mix}}}{\lambda^2 \mu^2 H} + \frac{G_1^2 \tau_{\mathrm{mix}}^2}{\mu^2} \epsilon_{\mathrm{app}} \Bigg) \nonumber \\
    &\quad + \alpha \mathcal{O} \left( \left( 1 + \frac{4}{\delta_s^2} \right) \frac{G_1^2 \tau_{\mathrm{mix}}^2}{\mu^2} \right) + \tilde{\mathcal{O}}\left(\frac{G_1c_\gamma \sqrt{\log(H / \delta)}}{\mu \lambda \sqrt{H}} + \beta\right) \\
    &\leq \tilde{\mathcal{O}}\left(\sqrt{\epsilon_{\mathrm{bias}}} + \sqrt{\epsilon_{\mathrm{app}}} + \frac{1}{\alpha K} + \alpha + \beta + \frac{\sqrt{\log(H / \delta)}}{\sqrt{H}} + m^{-1/4}\right)
    \end{align}
    From Equation 79 in \citet{xu2025global},
    \begin{align}
        \frac{1}{K} \mathbb{E}\left[\sum_{k=0}^{K-1} J_c(\theta_k) \left(\lambda_k - \frac{2}{\delta_s}\right)\right] \leq \frac{2}{\delta_s^2 \beta K} + \frac{\beta}{2} 
    \end{align}
    Since \(\lambda_k J_c^{\pi^*} \geq 0\), we have
    \begin{align}
    &\frac{1}{K} \mathbb{E} \left[\sum_{k=0}^{K-1} \left( J_r^{\pi^*} - J_r(\theta_k) \right) \right] + \frac{2}{\delta_s K} \mathbb{E} \left[\sum_{k=0}^{K-1} - J_c(\theta_k) \right] \nonumber \\
    &\leq \sqrt{\epsilon_{\mathrm{bias}}} + \frac{2}{\delta_s^2 \beta K} + \frac{\beta}{2} + \frac{1}{\alpha K} \mathbb{E}_{s \sim d^{\pi^*}} \left[ KL(\pi^*(\cdot|s) \| \pi_{\theta_0}(\cdot|s)) \right] \nonumber \\
    &\quad + \left(1 + \frac{2}{\delta_s}\right) \tilde{\mathcal{O}}\Bigg(\frac{G_1^3 \tau_{\mathrm{mix}}^2}{\sqrt{T}\mu^2} + \frac{G_1 \sqrt{\log(H / \delta)} \tau_{\mathrm{mix}}}{\mu \sqrt{T}} + \frac{G_1}{\mu m^{1/4}} + \frac{G_1 \tau_{\mathrm{mix}}}{\mu}\sqrt{\epsilon_{\mathrm{app}}} + \frac{G_1 c_\gamma \sqrt{\log(H / \delta)}}{\mu \lambda \sqrt{H}}\Bigg) \nonumber \\ 
    &\quad + \alpha G_2 \left(1 + \frac{4}{\delta_s^2}\right)\tilde{\mathcal{O}}\Bigg(\frac{G_1^6 \tau_{\mathrm{mix}}^3}{H\mu^4} + \frac{G_1^2 \log(H / \delta) \tau_{\mathrm{mix}}}{H\mu^2} + \frac{G_1^2}{\mu^2 \sqrt{m}} + \frac{G_1^2 c_\gamma^2 \log(H / \delta) \tau_{\mathrm{mix}}}{\lambda^2 \mu^2 H} + \frac{G_1^2 \tau_{\mathrm{mix}}^2}{\mu^2} \epsilon_{\mathrm{app}} \Bigg) \nonumber \\
    &\quad + \alpha \mathcal{O} \left( \left( 1 + \frac{4}{\delta_s^2} \right) \frac{G_1^2 \tau_{\mathrm{mix}}^2}{\mu^2} \right) 
    \end{align}
    Then by Lemma G.6 in \citet{xu2025global}, we have
    \begin{align}
    &\frac{1}{K} \mathbb{E} \left[\sum_{k=0}^{K-1} - J_c(\theta_k) \right] \nonumber \\
    &\leq \sqrt{\epsilon_{\mathrm{bias}}} + \frac{1}{\delta_s \beta K} + \frac{\delta_s \beta}{4} + \frac{\delta_s}{2\alpha K} \mathbb{E}_{s \sim d^{\pi^*}} \left[ KL(\pi^*(\cdot|s) \| \pi_{\theta_0}(\cdot|s)) \right] \nonumber \\
    &\quad + \left(1 + \frac{\delta_s}{2}\right) \tilde{\mathcal{O}}\Bigg(\frac{G_1^3 \tau_{\mathrm{mix}}^2}{\sqrt{T}\mu^2} + \frac{G_1 \sqrt{\log(H / \delta)} \tau_{\mathrm{mix}}}{\mu \sqrt{T}} + \frac{G_1}{\mu m^{1/4}} + \frac{G_1 \tau_{\mathrm{mix}}}{\mu}\sqrt{\epsilon_{\mathrm{app}}} + \frac{G_1 c_\gamma \sqrt{\log(H / \delta)}}{\mu \lambda \sqrt{H}}\Bigg) \nonumber \\ 
    &\quad + \alpha G_2 \left(\frac{\delta_s}{2} + \frac{2}{\delta_s}\right)\tilde{\mathcal{O}}\Bigg(\frac{G_1^6 \tau_{\mathrm{mix}}^3}{H\mu^4} + \frac{G_1^2 \log(H / \delta) \tau_{\mathrm{mix}}}{H\mu^2} + \frac{G_1^2}{\mu^2 \sqrt{m}} + \frac{G_1^2 c_\gamma^2 \log(H / \delta) \tau_{\mathrm{mix}}}{\lambda^2 \mu^2 H} + \frac{G_1^2 \tau_{\mathrm{mix}}^2}{\mu^2} \epsilon_{\mathrm{app}} \Bigg) \nonumber \\
    &\quad + \alpha \mathcal{O} \left( \left(\frac{\delta_s}{2} + \frac{2}{\delta_s} \right) \frac{G_1^2 \tau_{\mathrm{mix}}^2}{\mu^2} \right) \\
    &\leq \tilde{\mathcal{O}}\left(\sqrt{\epsilon_{\mathrm{bias}}} + \sqrt{\epsilon_{\mathrm{app}}} + \frac{1}{\alpha K} + \frac{1}{\beta K} + \alpha + \beta + \frac{\sqrt{\log(H / \delta)}}{\sqrt{H}} + m^{-1/4}\right)
    \end{align}
    Using \(K = H = \Theta(T^{1/2})\), \(\alpha = \beta = \Theta(T^{-1/4})\), we have
    \begin{align}
        \frac{1}{K} \mathbb{E} \left[\sum_{k=0}^{K-1} \left( J_r^{\pi^*} - J_r(\theta_k) \right)\right] &\leq \tilde{\mathcal{O}}\left(\sqrt{\epsilon_{\mathrm{bias}}} + \sqrt{\epsilon_{\mathrm{app}}} + T^{-1/4} + m^{-1/4}\right) \\
        \frac{1}{K} \mathbb{E} \left[\sum_{k=0}^{K-1} - J_c(\theta_k) \right] &\leq \tilde{\mathcal{O}}\left(\sqrt{\epsilon_{\mathrm{bias}}} + \sqrt{\epsilon_{\mathrm{app}}} + T^{-1/4} + m^{-1/4}\right)
    \end{align}
\end{proof}

%% file: Sections/Appendix/limitations_future_work.tex
\section{Limitations and Future Work}
While our work establishes strong theoretical guarantees for neural actor-critic methods in CMDPs, it presents several limitations that open avenues for future research:

\paragraph{Beyond the NTK Regime} Our theoretical analysis heavily relies on the Neural Tangent Kernel (NTK) regime, which requires the neural network's width $m$ to be highly overparameterized to ensure the parameters remain close to their initialization. While this allows for mathematically tractable linearized dynamics, it limits the network's ability to perform deep feature representation learning. A vital future direction is to extend this analysis beyond the lazy training regime, potentially utilizing mean-field theory or representation learning frameworks to capture the true feature-learning capabilities of deep RL.

\paragraph{Ergodicity and Unichain Assumptions} Our theoretical analysis relies on the uniform ergodicity assumption Assumption~\ref{assump:ergodic}, which necessitates that the Markov chain induced by every parameterized policy is irreducible and aperiodic. In many practical safe RL domains, such as autonomous driving or robotics, policies may naturally lead to absorbing states (e.g., system failures) or possess disjoint transient regions. Currently, only MDPs with linear critics have been developed in the unichain setting \citep{ganesh2025regret,satheesh2026regret}. Extending neural critic approaches to these settings remains an open problem.

\paragraph{Order-Optimal Convergence Rates}
We established a global convergence rate and cumulative constraint violation of $\tilde{\mathcal{O}}(T^{-1/4})$, which is not order-optimal natural actor-critic methods under Markovian sampling. The main bottleneck for CMDPs is the NTK analysis for the squared bias due to the projection operator. Extending our results by alleviating this bottleneck remains an interesting open problem.